%% file: main.tex
\documentclass[10pt,twocolumn,letterpaper]{article}

\usepackage{iccv}
\usepackage{times}
\usepackage{epsfig}
\usepackage{graphicx}
\usepackage{amsmath}
\usepackage{amssymb}
\usepackage[utf8]{inputenc} 
\usepackage{url}
\usepackage{booktabs}       
\usepackage{amsfonts}       
\usepackage{nicefrac}       
\usepackage{microtype}      
\usepackage[dvipsnames]{xcolor}         
\usepackage{algorithm}
\usepackage{algpseudocode}
\algrenewcommand\algorithmicrequire{\textbf{Input:}}
\algrenewcommand\algorithmicensure{\textbf{Output:}}
\usepackage{amsthm}
\newtheorem{proposition}{Proposition}
\usepackage{subcaption}
\usepackage{gensymb}
\usepackage{placeins}
\usepackage{multirow}
\usepackage{mathtools}
\input{math_commands.tex}

\usepackage[numbers]{natbib}

\usepackage[pagebackref=false,breaklinks=true,letterpaper=true,colorlinks,hidelinks,bookmarks=false]{hyperref}

\iccvfinalcopy 

\def\iccvPaperID{3238} 

\begin{document}

\title{Get the Best of Both Worlds: Improving Accuracy and Transferability by Grassmann Class Representation}

\author{Haoqi Wang\(^2\)\textsuperscript{\textasteriskcentered}\quad Zhizhong Li\(^1\)\thanks{~Equal contribution. Work is done when Haoqi was at SenseTime.}\textsuperscript{\textasteriskcentered}~\quad Wayne Zhang\(^{13}\)\thanks{~Corresponding author: Wayne Zhang.}\\
\(^1\) SenseTime Research \quad \(^2\) EPFL, Lausanne, Switzerland\\
\(^3\) Guangdong Provincial Key Laboratory of Digital Grid Technology
\\
{\tt\small haoqi.wang@epfl.ch \quad \{lizz,wayne.zhang\}@sensetime.com}
}


\maketitle
\ificcvfinal\thispagestyle{empty}\fi

\begin{abstract}
  We generalize the class vectors found in neural networks to linear subspaces (i.e.~points in the Grassmann manifold) and show that the Grassmann Class Representation (GCR) enables the simultaneous improvement in accuracy and feature transferability.
  In GCR, each class is a subspace and the logit is defined as the norm of the projection of a feature onto the class subspace.
  We integrate Riemannian SGD into deep learning frameworks such that class subspaces in a Grassmannian are jointly optimized with the rest model parameters.
  Compared to the vector form, the representative capability of subspaces is more powerful.
  We show that
  on ImageNet-1K, the top-1 error of ResNet50-D, ResNeXt50, Swin-T and Deit3-S are reduced by 5.6\%, 4.5\%, 3.0\% and 3.5\%, respectively.
  Subspaces also provide freedom for features to vary and we observed that the intra-class feature variability grows when the subspace dimension increases.
  Consequently, we found the quality of GCR features is better for downstream tasks.
  For ResNet50-D, the average linear transfer accuracy across 6 datasets improves from 77.98\% to 79.70\% compared to the strong baseline of vanilla softmax.
  For Swin-T, it improves from 81.5\% to 83.4\% and for Deit3, it improves from 73.8\% to 81.4\%.
  With these encouraging results, we believe that more applications could benefit from the Grassmann class representation.
  Code is released at \url{https://github.com/innerlee/GCR}.
\end{abstract}


\section{Introduction}\label{sec:intro}

The scheme \texttt{deep feature}\(\rightarrow\)\texttt{fully-connected} \(\rightarrow\)\texttt{softmax}\(\rightarrow\)\texttt{cross-entropy loss} has been the standard practice in deep classification networks.
Columns of the weight parameter in the fully-connected layer are the class representative vectors and serve as the prototype for classes.
The vector class representation has achieved huge success, yet it is not without imperfections.
In the study of transferable features, researchers noticed a dilemma that representations with higher classification accuracy lead to less transferable features for downstream tasks~\cite{kornblith2021why}.
This is connected to the fact that they tend to collapse intra-class variability of features, resulting in loss of information in the logits about the resemblances between instances of different classes~\cite{muller2019does}.
The neural collapse phenomenon~\cite{papyan2020prevalence} indicates that as training progresses, the intra-class variation becomes negligible, and features collapse to their class means.
As such, this dilemma inherently originates from the practice of representing classes by a single vector.
This motivates us to study representing classes by high-dimensional subspaces.

Representing classes as subspaces in machine learning can be dated back, at least, to 1973~\cite{watanabe1973subspace}.
This core idea is re-emerging recently in various contexts such as clustering~\cite{zhang2018scalable}, few-shot classification~\cite{devos2020regression,simon2020adaptive} and out-of-distribution detection~\cite{wang2022vim}, albeit in each case a different concrete instantiation was proposed.
However, very few works study the subspace representation in large-scale classification, a fundamental computer vision task that benefits numerous downstream tasks.
We propose the \emph{Grassmann Class Representation} (GCR) to fill this gap and study its impact on classification and feature transferability via extensive experiments.
To be specific, each class \(i\) is associated with a linear subspace \(S_i\), and for any feature vector \(\vx\), the \(i\)-th logit \(l_i\) is defined as the norm of its projection onto the subspace \(S_i\),
\begin{equation}\label{eq:firsteq}
  l_i := \left\|\proj_{S_i}\vx\right\|.
\end{equation}
In the following, we answer the two critical questions,
\begin{enumerate}
  \setlength\itemsep{-0.4em}
  \item How to effectively optimize the subspaces in training?
  \item Is Grassmann class representation useful?
\end{enumerate}

Several drawbacks and important differences in previous works make their methodologies hard to generalize to the large-scale classification problem.
Firstly, their subspaces might be not learnable.
In ViM~\cite{wang2022vim}, DSN~\cite{simon2020adaptive} and the SVD formulation of~\cite{zhang2018scalable}, subspaces are obtained \emph{post hoc} by PCA-like operation on feature matrices without explicit parametrization and learning.
Secondly, for works with learnable subspaces, their learning procedure for subspaces might not apply.
For example, in RegressionNet~\cite{devos2020regression}, the loss involves \emph{pairwise} subspace orthogonalization, which does not scale when the number of classes is large because the computational cost will soon be infeasible.
And thirdly, the objective of~\cite{zhang2018scalable} is unsupervised subspace clustering, which needs substantial changes to adapt to classification.

It is well known that the set of \(k\)-dimensional linear subspaces form a Grassmann manifold, so finding the optimal subspace representation for classes is to optimize on the Grassmannian.
Therefore, a natural solution to Question 1 is to use geometric optimization~\cite{edelman1998geometry}, which optimizes the objective function under the constraint of a given manifold.
Points being optimized are moving along geodesics instead of following the direction of Euclidean gradients.
We implemented an efficient Riemannian SGD for optimization in the Grassmann manifold in Algorithm~\ref{alg:gsgd},
which integrates the geometric optimization into deep learning frameworks so that the subspaces in Grassmannian and the model weights in Euclidean are jointly optimized.

The Grassmann class representation sheds light on the incompatibility issue between accuracy and transferability.
Features can vary in a high-dimensional subspace without harming the accuracy.
We empirically verify this speculation in Section~\ref{sec:exp}, which involves both CNNs (ResNet~\cite{he2016deep}, ResNet-D~\cite{he2019bag}, ResNeXt~\cite{xie2017aggregated}, VGG13-BN~\cite{simonyan2014very}) and vision transformers (Swin~\cite{liu2021swin} and Deit3~\cite{touvron2022deit}).
We found that with larger subspace dimensions \(k\), the intra-class variation increase, and the feature transferability improve.
The classification performance of GCR is also superior to the vector form.
For example, on ImageNet-1K, the top-1 error rates of ResNet50-D, ResNeXt50, Swin-T and Deit3-S are reduced relatively by 5.6\%, 4.5\%, 3.0\%, and 3.5\%, respectively.

To summarize, our contributions are three folds.
(1) We propose the Grassmann class representation and learn the subspaces jointly with other network parameters with the help of Riemannian SGD.
(2) We showed its superior accuracy on large-scale classification both for CNNs and vision transformers.
(3) We showed that features learned by the Grassmann class representation have better transferability.

\section{Related Work}\label{sec:related}

\paragraph{Geometric Optimization}
\cite{edelman1998geometry} developed the geometric Newton and conjugate gradient algorithms on the Grassmann and Stiefel manifolds in their seminal paper.
Riemannian SGD was introduced in~\cite{bonnabel2013stochastic} with an analysis on convergence and there are variants such as Riemannian SGD with momentum \cite{roy2018geometry} or adaptive \cite{kasai2019riemannian}.
Other popular Euclidean optimization methods such as Adam are also studied in the Riemannian manifold context~\cite{becigneul2018riemannian}.
\cite{lezcano2019cheap} study the special case of \(\SO(n)\) and \(U(n)\) and uses the exponential map to enable Euclidean optimization methods for Lie groups.
The idea was generalized into trivialization in~\cite{lezcano2019trivializations}.
Our Riemannian SGD Algorithm~\ref{alg:gsgd} is tailored for Grassmannian, so we use the closed-form equation for geodesics.
Applications of geometric optimization include matrix completion~\cite{7039534,Li_2015_CVPR,7282649,nimishakavi2018dual}, hyperbolic taxonomy embedding~\cite{nickel2018learning}, to name a few.
\cite{hamm2008grassmann} proposed the Grassmann discriminant analysis, in which features are modeled as linear subspaces.

\paragraph{Orthogonal Constraints}
Geometric optimization in deep learning is mainly used for providing orthogonal constraints in the design of network structure~\cite{harandi2016generalized,ozay2018training}, aiming to mitigate the gradient vanishing or exploding problems.
Orthogonality are also enforced via regularizations~\cite{arjovsky2016unitary,xie2017all,bansal2018can,qi2020deep,wang2020orthogonal}.
Contrastingly, we do not change the network structures, and focus ourselves on the subspace form of classes.
SiNN~\cite{roy2019siamese} uses the Stiefel manifold to construct Mahalanobis distance matrices in Siamese networks to improve embeddings in metric learning.
It does not have the concept of classes.

\paragraph{Improving Feature Diversity}
Our GCR favors the intra-class feature variation by providing a subspace to vary.
There are other efforts to encourage feature diversity.
SoftTriplet loss~\cite{qian2019softtriple} and SubCenterArcFace~\cite{deng2020sub} model each class as local clusters with several centers or sub-centers.
\cite{zhang2017learning} uses a global orthogonal regularization to drive local descriptors spread out in the features space.
\cite{yu2020learning} proposes to learn low-dimensional structures from the maximal coding rate reduction principle.
The subspaces are estimated using PCA on feature vectors after the training.

\paragraph{Classes as Subspaces}
ViM~\cite{wang2022vim} uses a subspace to denote the out-of-distribution class, which is obtained via PCA-like postprocessing after training.
\(k\)SCN~\cite{zhang2018scalable} uses subspaces to model clusters in unsupervised learning.
Parameters of models and subspaces are optimized alternatively in a wake-and-sleep fashion.
CosineSoftmax~\cite{kornblith2021why}
defines logits via the inner product between the feature and normalized class vector.
Since the class vector is normalized to be unit length, it is regarded as representing the class as a 1-dimensional subspace.
ArcFace~\cite{deng2019arcface} improves over cosine softmax by adding angular margins to the loss.
RegressionNet~\cite{devos2020regression} uses the subspace spanned by the \(K\) feature vectors of each class in the \(N\)-way \(K\)-shot classification.
The computational cost of its pairwise subspace orthogonalization loss is quadratic \emph{w.r.t.} the number of classes and becomes infeasible when the number of classes is large.
DSN~\cite{simon2020adaptive} for few-shot learning computed subspaces from the data matrix rather than parametrized and learned,
and its loss also involves pairwise class comparison which does not scale.
Different from these formulations, we explicitly parametrize classes as high-dimensional subspaces and use geometric optimization to learn them in supervised learning.

\section{Preliminaries}\label{sec:preliminary}

In this section, we briefly review the essential concepts in geometric optimization.
Detailed exposition can be found in~\cite{edelman1998geometry,absil2009optimization}.
Given an \(n\)-dimensional Euclidean space \(\R^n\), the set of \(k\)-dimensional linear subspaces forms the Grassmann manifold \(\gG(k, n)\).
A computational-friendly representation for subspace \(S \in\gG(k, n)\) is an orthonormal matrix \(\mS\in\R^{n\times k}\), where \(\mS^T \mS = \mI_k\) and \(\mI_k\) is the \(k\times k\) identity matrix.
Columns of the matrix \(\mS\) can be interpreted as an orthonormal basis for the subspace \(S\).
The matrix form is \emph{not unique}, as right multiplying an orthonormal matrix will produce a new matrix representing the same subspace.
Formally, Grassmannian is a quotient space of the Stiefel manifold and the orthogonal group \(\gG(k,n)=\St(k, n)/\gO(k)\), where \(\St(k, n)=\{\mX\in\R^{n\times k}|\mX^T \mX = \mI_k\}\) and \(\gO(k)=\{\mX\in\R^{k\times k}|\mX^T \mX = \mI_k\}\).
When the context is clear, we use the space \(S\) and one of its matrix forms \(\mS\) interchangeably.

Given a function \(f:\gG(k,n)\to\R\) defined on the Grassmann manifold,
the Riemannian gradient of \(f\) at point \(S\in\gG(k,n)\) is given by~\cite[Equ.~(2.70)]{edelman1998geometry},
\begin{equation}\label{eq:formula1}
  \nabla f(\mS)
  = f_{\mS} - {\mS}{\mS}^T f_{\mS},
\end{equation}
where \(f_{\mS}\) is the Euclidean gradient with elements \((f_{\mS})_{ij} = \frac{\partial f}{\partial \mS_{ij}}\).
When performing gradient descend on the Grassmann manifold,
and suppose the current point is \(\mS\) and the current Riemannian gradient is \(\mG\), then the next point is the endpoint of \(\mS\) moving along the geodesic toward the tangent \(\mG\) with step size \(t\).
The geodesic is computed by~\cite[Equ.~(2.65)]{edelman1998geometry},
\begin{equation}\label{eq:formula2}
  \mS(t) = \left(\mS\mV \cos(t\mSigma) + \mU\sin(t\mSigma)\right)\mV^T,
\end{equation}
where \(\mU\mSigma \mV^T = \mG\) is the thin SVD of \(\mG\).

\section{Learning Grassmann Class Representation}\label{sec:method}

Denote the weight of the last fully-connected (fc) layer in a classification network by \(\mW\in\R^{n\times C}\) and the bias by \(\vb\in\R^C\), where \(n\) is the dimension of features and \(C\) is the number of classes.
The \(i\)-th column vector \(\vw_i\) of \(\mW\) is called the \(i\)-th class representative vector.
The \(i\)-th logit is computed as the inner product between a feature \(\vx\) and the class vector (and optionally offset by a bias \(b_i\)), namely \(\vw_i^T \vx + b_i\).
We extend this well-established formula to a multi-dimensional subspace form
\(l_i := \left\|\proj_{S_i}\vx\right\|\)
where \(S_i\in\gG(k, n)\) is a \(k\)-dimensional subspace in the \(n\)-dimensional feature space.
We call \(S_i\) the \(i\)-th \emph{class representative space}, or class space in short.
Comparing the new logit to the standard one,
the inner product of feature \(\vx\) with class vector is replaced by the norm of the subspace projection \(\proj_{S_i}\vx\) and the bias term is omitted.
We found that normalizing features to a constant length \(\gamma\) improves training.
Incorporating this, \Eqref{eq:firsteq} becomes
\begin{equation}\label{eq:newlogit}
  l_i := \left\|\proj_{S_i}\frac{\gamma\vx}{\|\vx\|}\right\|.
\end{equation}
We assume \(\vx\) has been properly normalized throughout this paper
so that we can simply use \Eqref{eq:firsteq} in the discussion.
We call this formulation of classes and logits the \emph{Grassmann Class Representation} (GCR).

The subspace class formulation requires two changes to an existing network.
Firstly, the last fc layer is replaced by the \emph{Grassmann fully-connected layer}, which transforms features to logits using \Eqref{eq:newlogit}.
Details can be found in Section~\ref{sec:use}.
Secondly, the optimizer is extended to process the new geometric layer, which is explained in Section~\ref{sec:optimize}.
Ultimately, parameters of the geometric layer are optimized using Riemannian SGD, while other parameters are simultaneously optimized using SGD, AdamW, or Lamb, \emph{etc}.

\subsection{Grassmann Class Representation}\label{sec:use}

Suppose for class \(i\in\left\{1, 2, \dots, C\right\}\), its subspace representation is \(S_i\in\gG(k_i, n)\), where the dimension \(k_i\) is a hyperparameter and is fixed during training.
The tuple of subspaces \((S_1, S_2, \dots, S_C)\) will be optimized in the product space
\(\gG(k_1, n) \times \gG(k_2, n) \times \cdots \times \gG(k_C, n)\).
Denote a matrix instantiation of \(S_i\) as \(\mS_i\in\R^{n\times k}\), where the column vectors form an orthonormal basis of \(S_i\), then we concatenate these matrices into a big matrix
\begin{equation}\label{eq:sconcat}
  \mS = [\mS_1 \;\mS_2\; \cdots \;\mS_C] \in\R^{n\times(k_1+k_2+\cdots+k_C)}.
\end{equation}
The matrix \(\mS\) consists of the parameters that are optimized numerically.
For a feature \(\vx\), the product \(\mS_i^T \vx\) gives the coordinate of \(\proj_{S_i}\vx\) under the orthonormal basis formed by the columns of \(\mS_i\).
By definition in~\Eqref{eq:firsteq}, the logit for class \(i\) and the (normalized) feature \(\vx\) is
\begin{equation}
  l_i = \left\|\proj_{S_i}\vx\right\|
  =  \left\|\mS_i^T \vx\right\|.
\end{equation}

\paragraph{Grassmann Fully-Connected Layer}

We implement the geometric fully-connected layer using the plain old fc layer.
The shape of the weight \(\mS\) is \(n\times(k_1+k_2+\cdots+k_C)\), as shown in~\Eqref{eq:sconcat}.
In the forward pass, the input feature is multiplied with the weight matrix to get a temporary vector \(\vt = \mS^T \vx\), then the first element of the output is the norm of the sub-vector \((t_1, \dots, t_{k_1})\), and the second element of the output is the norm of \((t_{k_1+1}, t_{k_1+2}, \dots, t_{k_1+k2})\), and so on.
If all \(k_i\)'s be the same value \(k\), as in our experiments, then the computation can be conveniently paralleled in one batch using tensor computation libraries.

\paragraph{Parameter Initialization}

Each matrix instantiation of the subspace should be initialized as an orthonormal matrix.
To be specific,
each block \(\mS_i\) of the weight \(\mS\) in~\Eqref{eq:sconcat} is orthonormal,
while the matrix \(\mS\) needs not be orthonormal.
For each block \(\mS_i\), we first fill them with standard Gaussian noises and then use \(\qf(\mS_i)\), namely the Q factor of its QR decomposition, to transform it to an orthonormal matrix.
The geometric optimization Algorithm~\ref{alg:gsgd} will ensure their orthonormality during training.

\subsection{Optimize the Subspaces}\label{sec:optimize}

\begin{figure}
  \centering
  \includegraphics[width=0.7\linewidth]{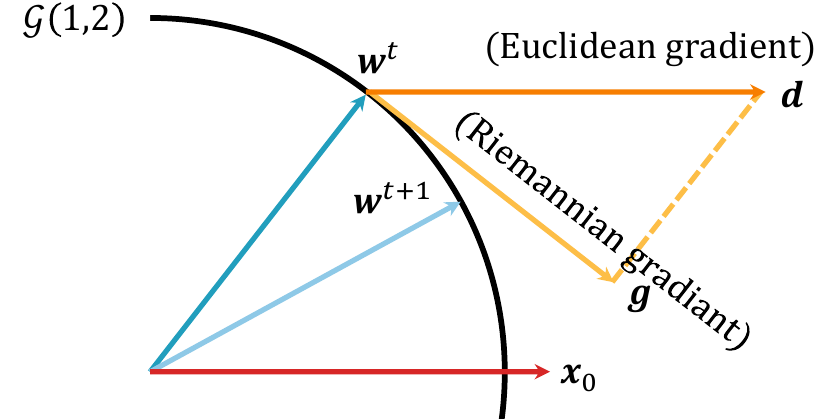}
  \caption{\small
    Geometric optimization in Grassmann manifold \(\gG(1, 2)\).
    Each point (\emph{e.g.} \(\vw^t\)) in the black circle represent the \(1\)-dimensional linear subspace \(S\) passing through it.
    The goal is to learn a subspace \(S\) to maximize \(\left\| \proj_S \vx_0 \right\|\).
    \(\vg\) is the Riemannian gradient obtained by the projection of Euclidean gradient \(\vd\).
    \(\vw^t\) moves along the geodesic towards the direction \(\vg\) to a new point \(\vw^{t+1}\).
  }\label{fig:toy1}
\end{figure}

Geometric optimization is to optimize functions defined on manifolds.
The key is to find the Riemannian gradient \emph{w.r.t.}\@ the loss function and then descend along the geodesic.
Here the manifold in concern is the Grassmannian \(\gG(k,n)\).
As an intuitive example, \(\gG(1, 2)\), composed of all lines passing through the origin in a two-dimensional plane,
can be pictured as a unit circle where each point on it denotes the line passing through that point.
Antipodal points represent the same line.
To illustrate how geometric optimization works, we define a toy problem on \(\gG(1, 2)\) that maximizes the norm of the projection of a fixed vector \(\vx_0\) onto a line through the origin, namely \(\max_{S\in\gG(1, 2)} \left\| \proj_S \vx_0 \right\|\).

As shown in Fig.~\ref{fig:toy1}, we represent \(S\) with a unit vector \(\vw\in S\).
Suppose at step \(t\), the current point is \(\vw^{(t)}\), then it is easy to compute that the Euclidean gradient at \(\vw^{(t)}\) is \(\vd=\vx_0\), and the Riemannian gradient \(\vg\) is the Euclidean gradient \(\vd\) projected to the tangent space of \(\gG(1,2)\) at point \(\vw^{(t)}\).
The next iterative point \(\vw^{(t+1)}\) is to move \(\vw^{(t)}\) along the geodesic toward the direction \(\vg\).
Without geometric optimization, the next iterative point would have lied at \(\vw^{(t)}+\gamma\vd\), jumping outside of the manifold.

The following proposition computes the Riemannian gradient for the subspace in \Eqref{eq:firsteq}.

\begin{proposition}\label{prop:grad}
  Let \(\mS\in\R^{n\times k}\) be a matrix instantiation of subspace \(S\in\gG(k,n)\), and \(\vx\in\R^n\) is a vector in Euclidean space, then the Riemannian gradient \(\mG\) of \(l(S, \vx) = \left\|\proj_{S}\vx\right\|\) w.r.t. \(S\) is
  \begin{equation}\label{eq:rgrad}
    \mG = \frac{1}{l} (\mI_n - \mS\mS^T) \vx \vx^T \mS.
  \end{equation}
\end{proposition}
\begin{proof}
  Rewrite \(\left\|\proj_{S}\vx\right\| = \sqrt{\vx^T \mS \mS^T \vx}\), and compute the Euclidean derivatives as
  \begin{equation}\label{eq:egrad}
    \frac{\partial l}{\partial \mS} = \frac{1}{l}\vx \vx^T \mS, \quad
    \frac{\partial l}{\partial \vx} = \frac{1}{l}\mS \mS^T \vx.
  \end{equation}
  Then~\Eqref{eq:rgrad} follows from~\Eqref{eq:formula1}.
\end{proof}

We give a geometric interpretation of Proposition~\ref{prop:grad}.
Let \(\vw_1\) be the unit vector along direction \(\proj_S \vx\), then expand it to an orthonormal basis of \(S\), say \(\{\vw_1, \vw_2, \dots, \vw_k\}\).
Since the Riemannian gradient is invariant to matrix instantiation, we can set \(\mS=[\vw_1 \;\vw_2\; \cdots \;\vw_k]\).
Then~\Eqref{eq:rgrad} becomes
\begin{equation}\label{eq:simpler}
  \mG = \left[\,(\mI_n - \mS\mS^T) \vx \;\;\,\vzero\;\; \cdots \;\;\vzero\,\right],
\end{equation}
since \(\vw_i\perp \vx, i=2, 3,\dots, k\) and \(\vw_1^T\vx = l\).
\Eqref{eq:simpler} shows that in the single-sample case, only one basis vector \(\vw_1\), the unit vector in \(S\) that is closest to \(\vx\), needs to be rotated towards vector \(\vx\).

\paragraph{Riemannian SGD}

\begin{algorithm}[t]
  \caption{An Iteration of the Riemannian SGD with Momentum for Grassmannian at Iteration \(t\)
  }\label{alg:gsgd}
  \begin{algorithmic}[1]
    \Require Learning rate \(\tau>0\), momentum \(\mu\in[0,1)\), Grassmannian weight matrix \(\mS^{(t)}\in\R^{n\times k}\), momentum buffer \(\mM^{(t-1)}\in\R^{n\times k}\), Euclidean gradient \(\mD\in\R^{n\times k}\).
    \vspace{0.15cm}
    \State Riemannian gradient by~\Eqref{eq:formula1}, \(\mG \gets (\mI_n - \mS\mS^T) \mD\).
    \State Approximately parallel transport \(\mM\) to the tangent space of current point \(\mS^{(t)}\) by projection
    \begin{equation}
      \mM \gets (\mI_n - \mS\mS^T) \mM^{(t-1)}.
    \end{equation}
    \State Update momentum \(\mM^{(t)} \gets \mu \mM + \mG\).
    \State Move along geodesic using~\Eqref{eq:formula2}.
    If \(\mU\mSigma \mV^T = \mM^{(t)}\) is the thin SVD, then
    \begin{equation*}
      \mS^{(t+1)} \gets \left(\mS^{(t)}\mV \cos(\tau\mSigma) + \mU\sin(\tau\mSigma)\right)\mV^T.
    \end{equation*}
    \State (Optional) Orthogonalization \(\mS^{(t+1)} \gets\qf(\mS^{(t+1)})\).
  \end{algorithmic}
\end{algorithm}

Parameters of non-geometric layers are optimized as usual using traditional optimizers such as SGD, AdamW, or Lamb during training.
For the geometric Grassmann fc layer, its parameters are optimized using the Riemannian SGD (RSGD) algorithm.
The pseudo-code of our implementation of RSGD with momentum is described in Algorithm~\ref{alg:gsgd}.
We only show the code for the single-sample, single Grassmannian case.
It is trivial to extend them to the batch version and the product of Grassmannians.
In step~2, we use projection to approximate the parallel translation of momentum, and the momentum update formula in step~3 is adapted from the official PyTorch implementation of SGD.
Weight decay does not apply here since spaces are scaleless.
Note that step 5 is optional since \(\mS^{(t+1)}\) in theory should be orthonormal.
In practice, to suppress the accumulation of numerical inaccuracies, we do an extra orthogonalization step using \(\qf(\cdot)\) every 5 iterations.
Algorithm~\ref{alg:gsgd} works seamlessly with traditional Euclidean optimizers and converts the gradient from Euclidean to Riemannian on-the-fly for geometric parameters.

\section{Experiment}\label{sec:exp}

In this section, we empirically study the influence of the Grassmann class representation under different settings.
In Section~\ref{sec:cls}, GCR demonstrates superior performance on the large-scale ImageNet-1K classification, a fundamental vision task.
We experimented with both CNNs and vision transformers and observed consistent improvements.
Then, in Section~\ref{sec:transfer}, we show that GCR improves the feature transferability by allowing larger intra-class variation.
The choice of hyper-parameters and design decisions are studied in Section~\ref{sec:analysis}.
Extra supportive experiments are presented in the supplementary material.

\paragraph{Experiment Settings}

For baseline methods, unless stated otherwise, we use the same training protocols (including the choice of batch size, learning rate policy, augmentation, optimizer, loss, and epochs) as in their respective papers.
The input size is \(224\times224\) for all experiments, and checkpoints with the best validation scores are used.
All codes, including the implementation of our algorithm and re-implementations of the compared baselines, are implemented based on the \emph{mmclassification}~\cite{mmcls} package.
PyTorch~\cite{NEURIPS2019_9015} is used as the training backend and each experiment is run on 8 NVIDIA Tesla V100 GPUs using distributed training.

Networks for the Grassmann class representation are set up by the drop-in replacement of the last linear fc layer in baseline networks with a Grassmann fc layer.
The training protocol is kept the same as the baseline whenever possible.
One necessary exception is to enhance the optimizer (\emph{e.g.}, SGD, AdamW or Lamb) with RSGD (\emph{i.e.}, RSGD+SGD, RSGD+AdamW, RSGD+Lamb) to cope with Grassmannian layers.
To reduce the number of hyper-parameters, we simply set the subspace dimension \(k\) to be the same for all classes and we use \(k=8\) throughout this section unless otherwise specified.
Suppose the dimension of feature space is \(n\), then the Grassmann fully-connected layer has the geometry of \(\Pi_{i=1}^{1000} \gG(8, n)\).
For hyper-parameters, we set \(\gamma=25\).
Experiments with varying \(k\)'s can be found in Section~\ref{sec:transfer} and experiments on tuning \(\gamma\) are discussed in Section~\ref{sec:analysis}.

\subsection{Improvements on Classification Accuracy}\label{sec:cls}

We apply Grassmann class representation to the large-scale classification task.
The widely used ImageNet-1K~\cite{5206848} dataset, containing 1.28M high-resolution training images and 50K validation images, is used to evaluate classification performances.
Experiments are organized into three groups which support the following observations.
(1) It has superior performance compared with different ways of representing classes.
(2) Grassmannian improves accuracy on different network architectures, including CNNs and the latest vision transformers.
(3) It also improves accuracy on different training strategies for the same architecture.

\begin{table}[t]
  \caption{
    Validation accuracy of ResNet50-D on ImageNet-1K using different class representations.
  }\label{tab:softtriple}
  \centering
  \small
  \setlength\tabcolsep{1.5pt}
  \begin{tabular}{l@{}ccl}
    \toprule
    Setting                                                              & Top1                   & Top5                   & Class Representation                                          \\
    \midrule
    Softmax~\cite{bridle1990probabilistic}                               & \(78.04\)              & \(93.89\)              & vector class representation                                   \\
    CosineSoftmax~\cite{kornblith2021why}                                & \(78.30\)              & \(94.07\)              & 1-dim subspace                                                \\
    ArcFace~\cite{deng2019arcface}                                       & \(76.66\)              & \(92.98\)              & \resizebox{0.95\width}{\height}{1-dim subspace with margin}   \\
    MultiFC                                                              & \(77.34\)              & \(93.65\)              & 8 fc layers ensembled                                         \\
    SoftTriple~\cite{qian2019softtriple}                                 & \(75.55\)              & \(92.62\)              & 8 centers weighted average                                    \\
    \resizebox{0.93\width}{\height}{SubCenterArcFace~\cite{deng2020sub}} & \(77.10\)              & \(93.51\)              & \resizebox{0.98\width}{\height}{8 centers with one activated} \\
    GCR (Ours)                                                           & \(\boldsymbol{79.26}\) & \(\boldsymbol{94.44}\) & \resizebox{0.98\width}{\height}{8-dim subspace with RSGD}     \\
    \bottomrule
  \end{tabular}
\end{table}

\begin{table*}[t]
  \caption{
    Comparing Grassmann class representation (\(k=8\)) with vector class representation on different architectures.
    Validation accuracy on ImageNet.
    \(n\) is the feature dimension, \emph{BS} means batch size, \emph{WarmCos} means using warm up together with the cosine learning rate decay.
    \emph{CE} is cross-entropy, \(LS\) is label smoothing, and \(BCE\) is binary cross-entropy.
  }\label{tab:arch}
  \centering
  \small
  \setlength\tabcolsep{2.8pt}
  \begin{tabular}{l@{}cc@{}c@{}l|rl@{}cc|rl@{}cc}
    \toprule
    \multicolumn{5}{c}{\textbf{Setting}} & \multicolumn{4}{c}{\textbf{Vector Class Representation}} & \multicolumn{4}{c}{\textbf{Grassmann Class Representation (\(k=8\))}}                                                                                                                                                                                                                                                         \\
    \textbf{Architecture}                & \(n\)                                                    & BS                                                                    & Epoch~ & Lr Policy & Loss & Optimizer~ & \textbf{Top1} & \textbf{Top5}~~ & Loss & Optimizer   & \textbf{Top1}                                                             & \textbf{Top5}                                                             \\
    \cmidrule(r){1-5}\cmidrule(lr){6-9}\cmidrule(l){10-13}
    ResNet50 \cite{he2016deep}           & 2048                                                     & 256                                                                   & 100    & Step      & CE   & SGD        & \(76.58\)     & \(93.05\)       & CE   & RSGD+SGD    & \(\boldsymbol{77.77}\){\footnotesize\color{OliveGreen}(\(\uparrow\)1.19)} & \(\boldsymbol{93.67}\){\footnotesize\color{OliveGreen}(\(\uparrow\)0.62)} \\
    ResNet50-D \cite{he2019bag}          & 2048                                                     & 256                                                                   & 100    & Cosine    & CE   & SGD        & \(78.04\)     & \(93.89\)       & CE   & RSGD+SGD    & \(\boldsymbol{79.26}\){\footnotesize\color{OliveGreen}(\(\uparrow\)1.22)} & \(\boldsymbol{94.44}\){\footnotesize\color{OliveGreen}(\(\uparrow\)0.55)} \\
    ResNet101-D \cite{he2019bag}~        & 2048                                                     & 256                                                                   & 100    & Cosine    & CE   & SGD        & \(79.32\)     & \(94.62\)       & CE   & RSGD+SGD    & \(\boldsymbol{80.24}\){\footnotesize\color{OliveGreen}(\(\uparrow\)0.92)} & \(\boldsymbol{94.95}\){\footnotesize\color{OliveGreen}(\(\uparrow\)0.33)} \\
    ResNet152-D \cite{he2019bag}         & 2048                                                     & 256                                                                   & 100    & Cosine    & CE   & SGD        & \(80.00\)     & \(95.02\)       & CE   & RSGD+SGD    & \(\boldsymbol{80.44}\){\footnotesize\color{OliveGreen}(\(\uparrow\)0.44)} & \(\boldsymbol{95.21}\){\footnotesize\color{OliveGreen}(\(\uparrow\)0.19)} \\
    ResNeXt50 \cite{xie2017aggregated}   & 2048                                                     & 256                                                                   & 100    & Cosine    & CE   & SGD        & \(78.02\)     & \(93.98\)       & CE   & RSGD+SGD    & \(\boldsymbol{79.00}\){\footnotesize\color{OliveGreen}(\(\uparrow\)0.98)} & \(\boldsymbol{94.28}\){\footnotesize\color{OliveGreen}(\(\uparrow\)0.30)} \\
    VGG13-BN \cite{simonyan2014very}     & 4096                                                     & 256                                                                   & 100    & Step      & CE   & SGD        & \(72.02\)     & \(90.79\)       & CE   & RSGD+SGD    & \(\boldsymbol{73.40}\){\footnotesize\color{OliveGreen}(\(\uparrow\)1.38)} & \(\boldsymbol{91.30}\){\footnotesize\color{OliveGreen}(\(\uparrow\)0.51)} \\
    Swin-T \cite{liu2021swin}            & 768                                                      & 1024                                                                  & 300    & WarmCos   & LS   & AdamW      & \(81.06\)     & \(95.51\)       & LS   & RSGD+AdamW~ & \(\boldsymbol{81.63}\){\footnotesize\color{OliveGreen}(\(\uparrow\)0.57)} & \(\boldsymbol{95.77}\){\footnotesize\color{OliveGreen}(\(\uparrow\)0.26)} \\
    Deit3-S \cite{touvron2022deit}       & 384                                                      & 2048                                                                  & 800    & WarmCos   & BCE  & Lamb       & \(81.53\)     & \(95.21\)       & CE   & RSGD+Lamb   & \(\boldsymbol{82.18}\){\footnotesize\color{OliveGreen}(\(\uparrow\)0.65)} & \(\boldsymbol{95.73}\){\footnotesize\color{OliveGreen}(\(\uparrow\)0.52)} \\
    \bottomrule
  \end{tabular}
\end{table*}

\paragraph{On Representing Classes}

In this group, we compare seven alternative ways to represent classes.
(1) \textbf{Softmax}~\cite{bridle1990probabilistic}
is the plain old vector class representation using the fc layer to get logits.
(2) \textbf{CosineSoftmax}~\cite{kornblith2021why} represents a class as a 1-dimensional subspace since the class vector is normalized to be unit length.
We set the scale parameter to \(25\) and do not add a margin.
(3) \textbf{ArcFace}~\cite{deng2019arcface} improves over cosine softmax by adding angular margins to the loss.
The default setting (\(s=64, m=0.5\)) is used.
(4) \textbf{MultiFC} is an ensemble of independent fc layers.
Specifically, we add \(8\) fc heads to the network.
These fc layers are trained side by side, and their losses are then averaged.
When testing, the logits are first averaged, and then followed by softmax to output the ensembled prediction.
(5) \textbf{SoftTriple}~\cite{qian2019softtriple} models each class by \(8\) centers.
The weighted average of logits computed from multiple class centers is used as the final logit.
We use the recommended parameters (\(\lambda=20, \gamma=0.1, \tau=0.2\) and \(\delta=0.01\)) from the paper.
(6) \textbf{SubCenterArcFace}~\cite{deng2020sub} improves over ArcFace by using \(K\) sub-centers for each class and in training only the center closest to a sample is activated.
We set \(K=8\) and do not drop sub-centers or samples since ImageNet is relatively clean.
(7) The last setting is our \textbf{GCR} with subspace dimension \(k=8\).
For all seven settings ResNet50-D is used as the backbone network and all models are trained on ImageNet-1K
using the same training strategy described in the second row of Tab.~\ref{tab:arch}.

Results are listed in Tab.~\ref{tab:softtriple},
from which we find that the Grassmann class representation is most effective.
Compared with the vector class representation of vanilla softmax, the top-1 accuracy improves from \(78.04\%\) to \(79.26\%\), which amounts to \(5.6\%\) relative error reduction.
Compared with previous ways of 1-dimensional subspace representation, \emph{i.e.\@} CosineSoftmax and ArcFace, our GCR improves the top-1 accuracy by \(0.96\%\) and \(2.60\%\), respectively.
Compared with the ensemble of multiple fc, the top-1 is improved by \(1.92\%\).
Interestingly, simply extending the class representation to multiple centers such as SoftTriple (\(75.55\%\)) and SubCenterArcFace (\(77.10\%\)) does not result in good performances when training from scratch on the ImageNet-1K dataset.
SoftTriple was designed for fine-grained classification and
SubCenterArcFace was designed for face verification.
Their strong performances in their intended domains do not naively generalize here.
This substantiates that making the subspace formulation competitive is a non-trivial contribution.

\paragraph{On Different Architectures}

We apply Grassmann class representation to eight network architectures, including six CNNs (ResNet50~\cite{he2016deep}, ResNet50/101/152-D~\cite{he2019bag}, ResNetXt50~\cite{xie2017aggregated}, VGG13-BN~\cite{simonyan2014very}) and two transformers (Swin~\cite{liu2021swin}, Deit3~\cite{touvron2022deit}).
For each model, we replace the last fc layer with Grassmannian fc and compare performances before and after the change.
Their training settings together with validation top-1 and top-5 accuracies are listed in Tab.~\ref{tab:arch}.
The results show that GCR is effective across different model architectures.
For all architectures, the improvement on top-1 is in the range \(0.44\!-\!1.38\%\).
The improvement is consistent not only for different architectures, but also across different optimizers (\emph{e.g.}, SGD, AdamW, Lamb) and different feature space dimensions (\emph{e.g.}, 2048 for ResNet, 768 for Swin, and 384 for Deit3).

\paragraph{On Different Training Strategies}

In this group,
we train ResNet50-D with the three training strategies (RSB-A3, RSB-A2, and RSB-A1) proposed in~\cite{wightman2021resnet}, which aim to push the performance of ResNets to the extreme.
Firstly, we train ResNet50-D with the original vector class representation and get top-1 accuracies of \(79.36\%\), \(80.29\%\), and \(80.53\%\), respectively.
Then, we replace the last classification fc with the Grassmann class representation (\(k=8\)), and their top-1 accuracies improve to \(79.88\%\), \(80.74\%\), and \(81.00\%\), respectively.
Finally, we add the FixRes~\cite{touvron2019fixing} trick to the three strategies, namely training on \(176\times176\) image resolution and when testing, first resize to \(232\times232\) and then center crop to \(224\times224\).
We get further boost in top-1 which are \(80.20\%\), \(81.04\%\) and \(81.29\%\), respectively.
Results are summarized in Fig.~\ref{fig:strategy}.

\begin{figure}[t]
  \centering
  \includegraphics[height=2.8cm]{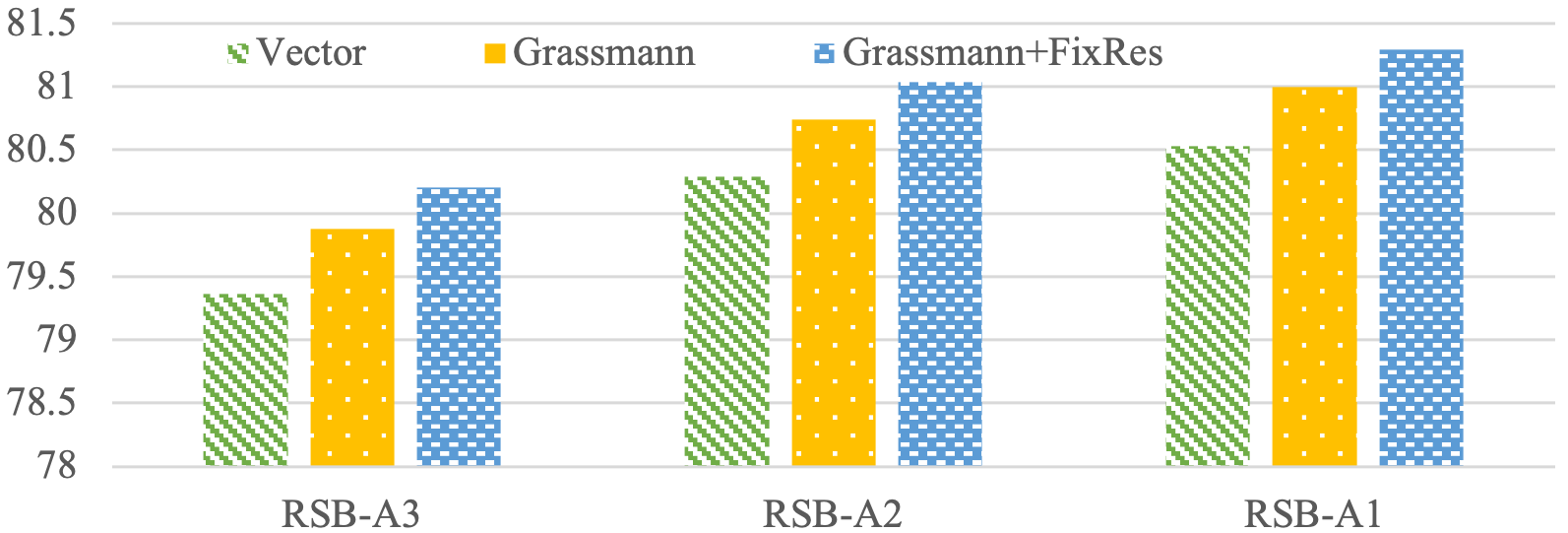}
  \caption{\small
    Validation accuracies of ResNet50-D on ImageNet-1K under different training strategies (RSB-A3, RSB-A2, and RSB-A1).
    Green bars are vector class representations;
    yellow bars are Grassmannian with \(k=8\);
    blue bars added the FixRes trick when training Grassmannian.
    The best top-1 of \textbf{ResNet50-D} is \textbf{81.29}\(\%\).
  }\label{fig:strategy}
\end{figure}

\subsection{Improvements on Feature Transferability}\label{sec:transfer}

In this section, we study the feature transferability of the Grassmann class representation.
Following~\cite{kornblith2021why} on the study of better losses \emph{vs.} feature transferability, we compare GCR with five different losses and regularizations.
They are Softmax~\cite{bridle1990probabilistic}, Cosine Softmax~\cite{kornblith2021why}, Label Smoothing~\cite{szegedy2016rethinking} (with smooth value 0.1), Dropout~\cite{srivastava2014dropout} (with drop ratio 0.3), and the Sigmoid~\cite{beyer2020we} binary cross-entropy loss.
Note that baselines in Tab.~\ref{tab:arch} that do not demonstrate competitive classification performances are not listed here.
The feature transfer benchmark dataset includes CIFAR-10~\cite{krizhevsky2009learning}, CIFAR-100~\cite{krizhevsky2009learning}, Food-101~\cite{bossard2014food}, Oxford-IIIT Pets~\cite{parkhi2012cats}, Stanford Cars~\cite{krause2013collecting}, and Oxford 102 Flowers~\cite{nilsback2008automated}.
All models are pre-trained on the ImageNet-1K dataset with the same training procedure as shown in the second row of Tab.~\ref{tab:arch}.
When testing on the transferred dataset, features (before the classification fc and Grassmann fc) of pre-trained networks are extracted.
We fit linear SVMs with the one-vs-rest multi-class policy on each of the training sets and report their top-1 accuracies or mean class accuracies (for Pets and Flowers) on their test set.
The regularization parameter for SVM is grid searched with candidates \([0.1, 0.2, 0.5, 1, 2, 5, 10, 15, 20]\) and determined by five-fold cross-validation on the training set.

\begin{table*}[t]
  \caption{Linear transfer using SVM for different losses.
    ResNet50-D is used as the backbone, and model weights are pre-trained on ImageNet-1K.
    \emph{Variability} measures the intra-class variability, and \(R^2\) measures class separation.
  }\label{tab:transfer}
  \centering
  \small
  \begin{tabular}{l@{}c|cc|c@{}c|c@{}c@{}ccccc}
    \toprule
    \multicolumn{2}{c}{\textbf{Setting}}        & \multicolumn{2}{c}{\textbf{ImageNet}} & \multicolumn{2}{c}{\textbf{Analysis}} & \multicolumn{6}{c}{\textbf{Linear Transfer (SVM)}}                                                                                                                          \\
    Name                                        & \(k\)                                 & Top-1                                 & Top-5                                              & Variability~ & \(R^2\)   & CIFAR10~  & CIFAR100~~ & ~Food     & Pets      & Cars      & Flowers   & \textbf{Avg.}      \\
    \cmidrule(r){1-2}\cmidrule(lr){3-4}\cmidrule(lr){5-6}\cmidrule(lr){7-12}\cmidrule(l){13-13}
    Softmax~\cite{bridle1990probabilistic}      &                                       & \(78.04\)                             & \(93.89\)                                          & \(60.12\)    & \(0.495\) & \(90.79\) & \(67.76\)  & \(72.13\) & \(92.49\) & \(51.55\) & \(93.17\) & \(77.98\)          \\
    CosineSoftmax~\cite{kornblith2021why}       &                                       & \(78.30\)                             & \(94.07\)                                          & \(56.87\)    & \(0.528\) & \(89.34\) & \(65.32\)  & \(64.79\) & \(91.68\) & \(43.92\) & \(87.28\) & \(73.72\)          \\
    LabelSmoothing~\cite{szegedy2016rethinking} &                                       & \(78.07\)                             & \(94.10\)                                          & \(54.79\)    & \(0.577\) & \(89.14\) & \(63.22\)  & \(66.02\) & \(91.72\) & \(43.58\) & \(91.01\) & \(74.12\)          \\
    Dropout~\cite{srivastava2014dropout}        &                                       & \(77.92\)                             & \(93.80\)                                          & \(55.40\)    & \(0.565\) & \(89.27\) & \(64.33\)  & \(66.74\) & \(91.38\) & \(43.99\) & \(88.59\) & \(74.05\)          \\
    Sigmoid~\cite{beyer2020we}                  &                                       & \(78.04\)                             & \(93.81\)                                          & \(60.20\)    & \(0.491\) & \(91.09\) & \(69.26\)  & \(71.71\) & \(91.98\) & \(51.75\) & \(92.86\) & \(78.11\)          \\
    \cmidrule(r){1-2}\cmidrule(lr){3-4}\cmidrule(lr){5-6}\cmidrule(lr){7-12}\cmidrule(l){13-13}
    \multirow{5}{*}{GCR (Ours)}                 & 1                                     & \(78.42\)                             & \(94.14\)                                          & \(56.50\)    & \(0.534\) & \(89.98\) & \(66.34\)  & \(64.34\) & \(91.37\) & \(42.97\) & \(86.85\) & \(73.64\)          \\
                                                & 4                                     & \(78.68\)                             & \(94.32\)                                          & \(61.48\)    & \(0.459\) & \(90.56\) & \(67.45\)  & \(67.58\) & \(91.37\) & \(50.24\) & \(90.08\) & \(76.21\)          \\
                                                & 8                                     & \(\mathbf{79.26}\)                    & \(\mathbf{94.44}\)                                 & \(63.49\)    & \(0.430\) & \(90.13\) & \(67.90\)  & \(70.06\) & \(91.85\) & \(53.25\) & \(92.64\) & \(77.64\)          \\
                                                & 16                                    & \(79.21\)                             & \(94.37\)                                          & \(65.79\)    & \(0.395\) & \(91.09\) & \(69.58\)  & \(71.28\) & \(91.99\) & \(55.93\) & \(93.80\) & \(78.95\)          \\
                                                & 32                                    & \(78.63\)                             & \(94.05\)                                          & \(67.74\)    & \(0.365\) & \(91.35\) & \(69.49\)  & \(71.80\) & \(92.47\) & \(58.05\) & \(95.04\) & \(\mathbf{79.70}\) \\
    \bottomrule
  \end{tabular}
\end{table*}

\begin{table}[t]
  \caption{Feature transfer using Swin-T and Deit3-S.
    All model weights are pre-trained on ImageNet-1K as in Tab.~\ref{tab:arch}.
    \emph{C10/100} is CIFAR10/100, \emph{Flwr} is Flowers.
    \emph{Swin-T GCR} and \emph{Deit3-S GCR} are their Grassmann variants.
  }\label{tab:transfer2}
  \centering
  \small
  \setlength\tabcolsep{2 pt}
  \begin{tabular}{l@{}ccccccccc}
    \toprule
    \multicolumn{1}{c}{\textbf{Setting}} & \multicolumn{2}{c}{\textbf{Analysis}} & \multicolumn{7}{c}{\textbf{Linear Transfer (SVM)}}                                                                                                                                                 \\
    Architecture~                        & Vari.\@                               & \(R^2\)                                            & C10      & \resizebox{0.92\width}{\height}{C100} & \resizebox{0.93\width}{\height}{Food} & Pets     & Cars     & Flwr     & \textbf{Avg.}     \\
    \cmidrule(r){1-1}\cmidrule(lr){2-3}\cmidrule(lr){4-9}\cmidrule(l){10-10}
    Swin-T                               & 60.2                                  & 0.48                                               & \(92.7\) & \(69.4\)                              & \(77.5\)                              & \(92.1\) & \(61.3\) & \(96.0\) & \(81.5\)          \\
    Swin-T GCR                           & 62.9                                  & 0.40                                               & \(93.5\) & \(71.5\)                              & \(79.8\)                              & \(93.3\) & \(65.5\) & \(97.0\) & \(\mathbf{83.4}\) \\
    \cmidrule(r){1-1}
    Deit3-S                              & 50.6                                  & 0.60                                               & \(89.5\) & \(63.7\)                              & \(64.7\)                              & \(91.4\) & \(43.1\) & \(90.2\) & \(73.8\)          \\
    Deit3-S GCR                          & 61.5                                  & 0.44                                               & \(93.0\) & \(71.9\)                              & \(74.9\)                              & \(92.3\) & \(60.7\) & \(95.5\) & \(\mathbf{81.4}\) \\
    \bottomrule
  \end{tabular}
\end{table}

\paragraph{Results}

The validation accuracies of different models on ImageNet-1K are listed in the second group of columns in Tab.~\ref{tab:transfer}.
All GCR models (\(k=1, 4, 8, 16, 32\)) achieve higher top-1 and top-5 accuracies than all the baseline methods with different losses or regularizations.
Within a suitable range, a larger subspace dimension \(k\) improves the accuracy greater.
However, when the subspace dimension is beyond 16, the top-1 accuracy begins to decrease.
When \(k=32\), the top-1 is \(78.63\%\), which is still \(0.33\%\) higher than the best classification baseline CosineSoftmax.

The linear transfer results are listed in the fourth group of columns in Tab.~\ref{tab:transfer}.
Among the baseline methods, we find that Softmax and Sigmoid have the highest average linear transfer accuracies, which are \(77.98\%\) and \(78.11\%\), respectively.
Other losses demonstrate worse transfer performance than Softmax.
For the Grassmann class representation, we observe a monotonic increase in average transfer accuracy when \(k\) increases from 1 to 32.
When \(k=1\), the cosine softmax and the GCR have both comparable classification accuracies and comparable transfer performance.
This can attribute to their resemblances in the formula.
The transfer accuracy of GCR (73.64\%) is lower than Softmax (77.98\%) at this stage.
Nevertheless, when the subspace dimension \(k\) increases, the linear transfer accuracy gradually improves, and when \(k=8\), the transfer performance (77.64\%) is on par with the Softmax.
When \(k\ge16\), the transfer performance surpasses all the baselines.

In Tab.~\ref{tab:transfer2}, we show that features of the GCR version of Swin-T and Deit3 increase the average transfer accuracy by \(1.9\%\) and \(7.6\%\), respectively.

\paragraph{Intra-Class Variability Increases with Dimension}

The intra-class variability is measured by first computing the mean pairwise angles (in degrees) between features within the same class and then averaging over classes.
Following the convention in the study of neural collapse~\cite{papyan2020prevalence}, the global-centered training features are used.
\cite{kornblith2021why} showed that alternative objectives which may improve accuracy over Softmax by collapsing the intra-class variability (see the \emph{Variability} column in Tab.~\ref{tab:transfer}), degrade the quality of features on downstream tasks.
Except for the Sigmoid, which has a similar intra-class variability (60.20) to Softmax (60.12), all other losses, including CosineSoftmax, LabelSmoothing, and Dropout, have smaller feature variability within classes (in the range from 54.79 to 56.87).
However, the above conclusion does not apply when the classes are modeled by subspaces.
For Grassmann class representation, we observed that if \(k\) is not extremely large, then \emph{as \(k\) increases, both the top-1 accuracy and the intra-class variability grow.}
This indicates that
representing classes as subspaces enables the simultaneous improvement of inter-class discriminability and intra-class variability.

This observation is also in line with the class separation index \(R^2\).
\(R^2\) is defined as one minus the ratio of the average intra-class cosine distance to the overall average cosine distance \cite[Eq.~(11)]{kornblith2021why}.
\cite{kornblith2021why} founds that greater class separation \(R^2\) is associated with less transferable features.
Tab.~\ref{tab:transfer} shows that when \(k\) increases, the class separation monotonically decreases, and the transfer performance grows accordingly.

\subsection{Design Choices and Analyses}\label{sec:analysis}

In this section, we use experiments to support our design choices and provide visualizations for the principal angles between class representative spaces.

\paragraph{Choice of Gamma}

In Tab.~\ref{tab:gamma}, we give more results with different values of \(\gamma\) when subspace dimension \(k=8\).
We find \(\gamma=25\) has good performance and use it throughout the paper without further tuning.

\begin{table}[t]
  \caption{
    Validation accuracy of Grassmann ResNet50-D on ImageNet-1K with varying \(\gamma\).
  }\label{tab:gamma}
  \centering
  \small
  \begin{tabular}{lcccc}
    \toprule
    Setting                         & \(k\)              & \(\gamma\) & Top1                   & Top5                   \\
    \cmidrule(r){1-3}\cmidrule(l){4-5}
    \multirow{3}{*}{ResNet50-D GCR} & \multirow{3}{*}{8} & \(20\)     & \(79.11\)              & \(94.29\)              \\
                                    &                    & \(25\)     & \(\boldsymbol{79.26}\) & \(\boldsymbol{94.44}\) \\
                                    &                    & \(30\)     & \(78.47\)              & \(94.07\)              \\
    \bottomrule
  \end{tabular}
\end{table}

\paragraph{Importance of Normalizing Features}

Normalizing the feature in \Eqref{eq:newlogit} is critical to the effective learning of the Grassmann class representations.
In Tab.~\ref{tab:renormalize} we compare results with/without feature normalization and observed a significant performance drop without normalization.

\begin{table}[t]
  \caption{
    Validation accuracy of Grassmann ResNet50-D on ImageNet with/without feature normalization.
  }\label{tab:renormalize}
  \centering
  \small
  \setlength\tabcolsep{5 pt}
  \begin{tabular}{lccccc}
    \toprule
    Setting                         & \(k\)              & Feature Normalize & Top1                   & Top5                   \\
    \cmidrule(r){1-3}\cmidrule(l){4-5}
    \multirow{2}{*}{ResNet50-D GCR} & \multirow{2}{*}{1} &                   & \(77.91\)              & \(93.78\)              \\
                                    &                    & \checkmark        & \(\boldsymbol{78.42}\) & \(\boldsymbol{94.14}\) \\
    \cmidrule(r){1-3}\cmidrule(l){4-5}
    \multirow{2}{*}{ResNet50-D GCR} & \multirow{2}{*}{8} &                   & \(78.12\)              & \(93.90\)              \\
                                    &                    & \checkmark        & \(\boldsymbol{79.26}\) & \(\boldsymbol{94.44}\) \\
    \bottomrule
  \end{tabular}
\end{table}

\paragraph{Principal Angles Between Class Representative Spaces}

When classes are subspaces, relationships between classes can be measured by \(k\) \emph{principal angles}, which contain richer information than a single angle between two class vectors.
The principal angles between two \(k\)-dimensional subspaces \(S\) and \(R\) are recursively defined as,
\begin{equation}
  \begin{aligned}
    \cos(\theta_i) & = \max_{\vs\in S}\max_{\vr\in R} \vs^T \vr = \vs_i^T \vr_i, \\ s.t. \|\vs\|=\|\vr\|=1,&~\vs^T\vs_j= \vr^T\vr_j=0, j\le i-1,
  \end{aligned}
\end{equation}
for \(i=1, \dots, k\) and \(\theta_i\in[0,\pi/2]\).
In Fig.~\ref{fig:principalangles}, we illustrate the smallest and largest principal angles between any pair of classes for a model with \(k=8\).
From the figure, we can see that the smallest principal angle reflects class similarity, and the largest principal angle is around \(\pi/2\).
A smaller angle means the two classes are correlated in some direction,
and a \(\pi/2\) angle means that some directions in one class subspace are completely irrelevant (orthogonal) to the other class.

\begin{figure}[t]
  \centering
  \makebox[0\textwidth][c]{
    \begin{subfigure}[t]{0.34\linewidth}
      \centering
      \includegraphics[height=3cm]{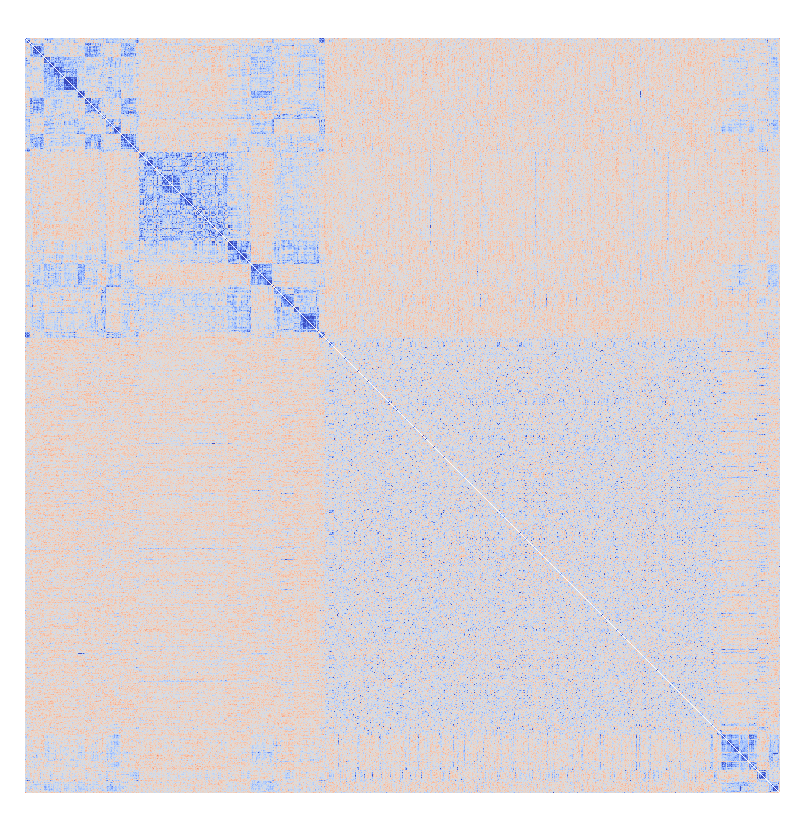}
      \caption{~}\label{fig:principalangle1}
    \end{subfigure}%
    \begin{subfigure}[t]{0.33\linewidth}
      \centering
      \includegraphics[height=3cm]{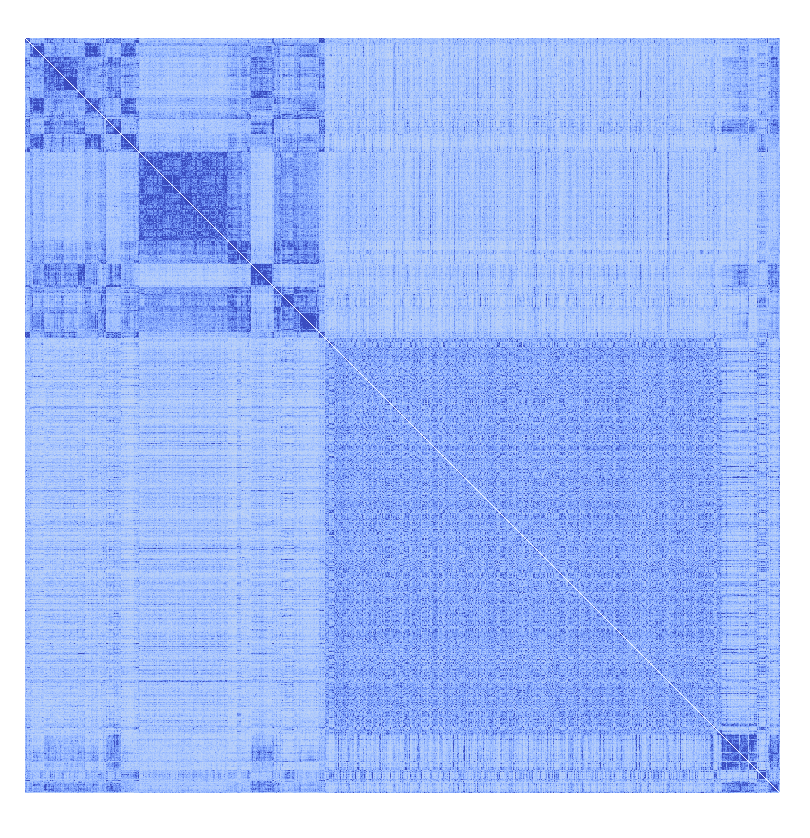}
      \caption{~}\label{fig:principalangle2}
    \end{subfigure}
    \begin{subfigure}[t]{0.32\linewidth}
      \centering
      \includegraphics[height=3cm]{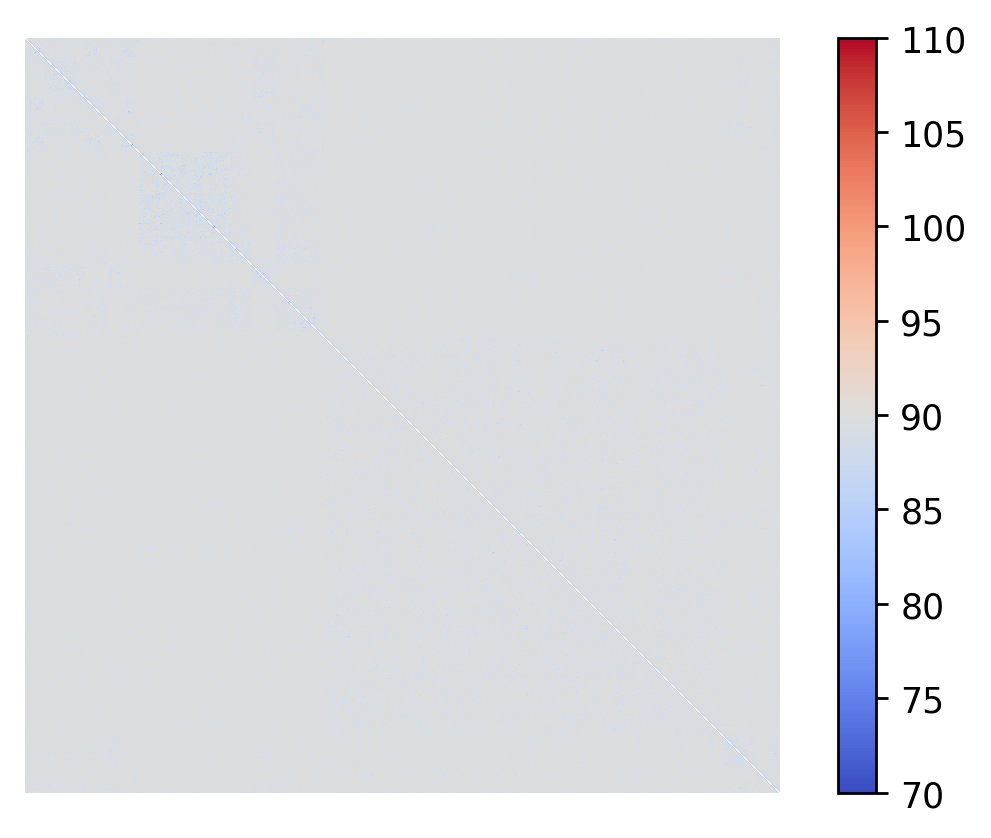}
      \caption{~}\label{fig:principalangle3}
    \end{subfigure}
  }
  \vspace{-0.5em}
  \caption{\small
    Each sub-figure is a heatmap of \(1000\times1000\) grids.
    The color at the \(i\)-th row and the \(j\)-th column represent an angle between class \(i\) and class \(j\) in ImageNet-1K.
    (a) Pairwise angles between class vectors of the ResNet50-D trained by vanilla softmax.
    Grids with red hue is large than \(90^{\circ}\), and blue hue means smaller than \(90^{\circ}\).
    (b) Pairwise smallest principal angles between 8-dimensional class subspaces of a ResNet50-D model.
    Deeper blue colors indicate smaller angles.
    (c) Pairwise largest principal angles of the same model as in (b).
    Grayish color means they are close to \(90^{\circ}\).
    Best viewed on screen with colors.
  }\label{fig:principalangles}
  \vspace{-0.5em}
\end{figure}

\paragraph{Necessity of Geometric Optimization}

To investigate the necessity of constraining the subspace parameters to lie in the Grassmannian, we replace the Riemannian SGD with the vanilla SGD and compare it with Riemannian SGD\@.
Note that with SGD, the logit formula \(\lVert\mS_i^T x\rVert\) no longer means the projection norm because \(\mS_i\) is not guaranteed to be orthonormal anymore.
With vanilla SGD, we get top-1 \(78.55\%\) and top-5 \(94.18\%\) when \(k=8\).
The top-1 is \(0.71\%\) lower than models trained by Riemannian SGD.

\section{Limitation and Future Direction}\label{sec:limitation}

Firstly, a problem that remains open is how to choose the optimal dimension.
Currently, we treat it as a hyper-parameter and decide it empirically.
Secondly, we showed that the Grassmann class representation \emph{allows for} greater intra-class variability.
Given this, it is attractive to explore extensions to \emph{explicitly promote} intra-class variability.
For example, a promising approach is to combine it with self-supervised learning.
We hope our work would stimulate progresses in this direction.

\section{Conclusion}\label{sec:conclusion}

In this work, we proposed the Grassmann class representation as a drop-in replacement of the conventional vector class representation.
Classes are represented as high-dimensional subspaces and the geometric structure of the corresponding Grassmann fully-connected layer is the product of Grassmannians.
We optimize the subspaces using the optimization and provide an efficient Riemannian SGD implementation tailored for Grassmannians.
Extensive experiments demonstrate that the new Grassmann class representation is able to improve classification accuracies on large-scale datasets and boost feature transfer performances at the same time.

  {\small
    \bibliographystyle{ieee_fullname}
    \bibliography{bib}
  }

\newpage

\appendix

\section{An Alternative Form of Riemannian SGD}\label{sec:retraction}

As discussed in Section~4.2, an important ingredient of the geometric optimization algorithm is to move a point in the direction of a tangent vector while staying on the manifold (\emph{e.g.}, see the example in Fig.~1).
This is accomplished by the \emph{retraction} operation (please refer to~\cite[Section~4.1]{absil2009optimization} for its mathematical definition).
In the fourth step of Alg.~1, the retraction is implemented by computing the geodesic curve on the manifold that is tangent to the vector \(\mM^{(t)}\).
An alternative implementation of retraction other than moving parameters along the geodesic is to replace step 4 with the Euclidean gradient update \(\mS^{(t+1)} \gets \mS^{(t)} + \tau\mM^{(t)} \) and then follow by the orthogonalization described in step~5.
In this case, step~5 is not optional anymore since \(\mS^{(t+1)}\) will move away from the Grassmannian after the Euclidean gradient update.
The orthogonalization pulls \(\mS^{(t+1)}\) back to the Grassmann manifold.
For ease of reference, we call this version of Riemannian SGD as \emph{Alg.~1 variant}.
We compare the two implementations in Tab.~\ref{tab:alg1variant}.
The results show that the Grassmann class representation is effective on both versions of Riemannian SGD.
We choose Alg.~1 because it is faster than the Alg.~1 variant.
The thin SVD used in Equ.~(3) can be efficiently computed via the \emph{gesvda} approximate algorithm provided by the cuSOLVER library, which is faster than a QR decomposition on GPUs (see Tab.~\ref{tab:svdtime}).

\begin{table}[h]
  \caption{
    Validation accuracy of Grassmann ResNet50-D on ImageNet with different retractions.
    The first row uses the exponential map, \emph{i.e.,} moving along geodesics, as retraction, while the second row uses the Q factor of QR decomposition, \emph{i.e.,} the \(\qf\) function, as retraction.
  }\label{tab:alg1variant}
  \centering
  \small
  \setlength{\tabcolsep}{5pt}
  \begin{tabular}{lllll}
    \toprule
    Setting        & Optimizer & Retraction & Top1                   & Top5                   \\
    \cmidrule(r){1-3}\cmidrule(l){4-5}
    Alg.~1         & RSGD+SGD  & Geodesic   & \(\boldsymbol{79.26}\) & \(\boldsymbol{94.44}\) \\
    Alg.~1 Variant & RSGD+SGD  & \(\qf\)    & \(79.13\)              & \(94.45\)              \\
    \bottomrule
  \end{tabular}
\end{table}

\section{Details on Step 5 of Algorithm~1}

The numerical inaccuracy is caused by the accumulation of tiny computational errors of Equ.~(3).
After running many iterations, the matrix \(\mS\) might not be perfectly orthogonal.
For example, after 100, 1000, and 5000 iterations of the Grassmannian ResNet50-D with subspace dimension \(k=8\), we observed that the error \(\max{}_i\lVert\mS_i^T\mS_i-\mI\rVert{}_\infty\) is 1.9e-5, 9.6e-5 and 3.7e-4, respectively.
After 50 epochs, the error accumulates to 0.0075.
So, we run step 5 every 5 iterations to keep both the inaccuracies and the extra computational cost at a low level at the same time.

\section{The Importance of Joint Training}

The joint training of the class subspaces and the features is essential.
To support this claim, we add an experiment (first row of Tab.~\ref{tab:finetune}) that only fine-tunes the class subspaces from weights pre-trained using the regular softmax.
We find that if the feature is fixed, changing the regular fc to the geometric version does not increase performance noticeably (top-1 from 78.04\% of the regular softmax version to 78.14\% of the Grassmann version).
For comparison, we also add another experiment that fine-tunes all parameters (second row of Tab.~\ref{tab:finetune}).
But when all parameters are free to learn, the pre-trained weights provide a good initialization that boosts the top-1 to 79.44\%.

\begin{table}[t]
  \caption{
    Validation accuracy of ResNet50-D on ImageNet trained with good initialization.
    Both rows use the weights of a ResNet50-D trained on ImageNet using the regular softmax.
    The first row fixes the backbone parameters, and solely learns the Grassmann fc layer, while the second row learns all parameters.
  }\label{tab:finetune}
  \centering
  \small
  \setlength\tabcolsep{3.4 pt}
  \begin{tabular}{llllcccc}
    \toprule
    Setting              & \(k\)              & Initialization                       & Fine-Tune  & Top1                   & Top5                   \\
    \cmidrule(r){1-4}\cmidrule(l){5-6}
    \multirow{2}{*}{GCR} & \multirow{2}{*}{8} & \multirow{2}{*}{Softmax pre-trained} & Last layer & \(78.14\)              & \(93.97\)              \\
                         &                    &                                      & All layers & \(\boldsymbol{79.44}\) & \(\boldsymbol{94.58}\) \\
    \bottomrule
  \end{tabular}
\end{table}

\section{Influence on Training Speed}

During training, the most costly operation in Alg.~1 is SVD\@.
The time of SVD and QR on typical matrix sizes encountered in an iteration of Alg.~1 is benchmarked in Tab.~\ref{tab:svdtime}.
We can see that (1) SVD (with gesvda solver) is faster QR decomposition, and (2), when the subspace dimension is no greater than 2, CPU is faster than GPU.
Based on these observations, in our implementation, we compute SVD on CPU when \(k\le2\) and on GPU in other cases.
Overall, when computing SVD, it adds roughly 5ms to 30ms overhead to the iteration time.

\begin{table}[t]
  \caption{
    SVD and QR time (in \emph{ms}) on Intel Xeon Gold 6146 (48) @ 3.201GHz and Nvidia V100 using PyTorch 1.13.1.
    The tested matrices are filled by Gaussian noises and have the shape of num classes \(\times\) feature dimension \(\times\) subspace dimension,
    which are the typical sizes encountered in Alg.~1 when training ImageNet-1K.
  }\label{tab:svdtime}
  \centering
  \small
  \begin{tabular}{lrrrr}
    \toprule
                                   & \multicolumn{2}{c}{\textbf{SVD}} & \multicolumn{2}{c}{\textbf{QR}}                         \\
    Tensor Shape                   & CPU                              & GPU                             & CPU       & GPU       \\
    \cmidrule(r){1-1}\cmidrule(lr){2-3}\cmidrule(l){4-5}
    \(1000 \times 2048 \times 1 \) & \(~~~~5.1\)                      & \(20.0\)                        & \(2.6\)   & \(36.5\)  \\
    \(1000 \times 2048 \times 2 \) & \(~~~~11.3\)                     & \(20.7\)                        & \(7.8\)   & \(46.0\)  \\
    \(1000 \times 2048 \times 4 \) & \(~~59.9\)                       & \(21.2\)                        & \(18.8\)  & \(56.3\)  \\
    \(1000 \times 2048 \times 8 \) & \(~~138.0\)                      & \(23.2\)                        & \(78.8\)  & \(70.9\)  \\
    \(1000 \times 2048 \times 16\) & \(352.9\)                        & \(25.2\)                        & \(200.5\) & \(118.5\) \\
    \(1000 \times 2048 \times 32\) & \(1020.1\)                       & \(30.8\)                        & \(626.8\) & \(224.6\) \\
    \bottomrule
  \end{tabular}
\end{table}

To measure the actual impact on training speed, we show the average iteration time (including a full forward pass and a full backward pass) of the vector class representation version \emph{vs.} the Grassmann class representation version on different network architectures in Fig.~\ref{fig:itertime}.
Overall, the Grassmann class representation adds about \(0.3\%\) (Deit3-S) to \(35.0\%\) (ResNet50-D) overhead.
The larger the model, and the large the batch size, the smaller the relative computational cost.

\begin{figure}[t]
  \centering
  \includegraphics[width=0.98\linewidth]{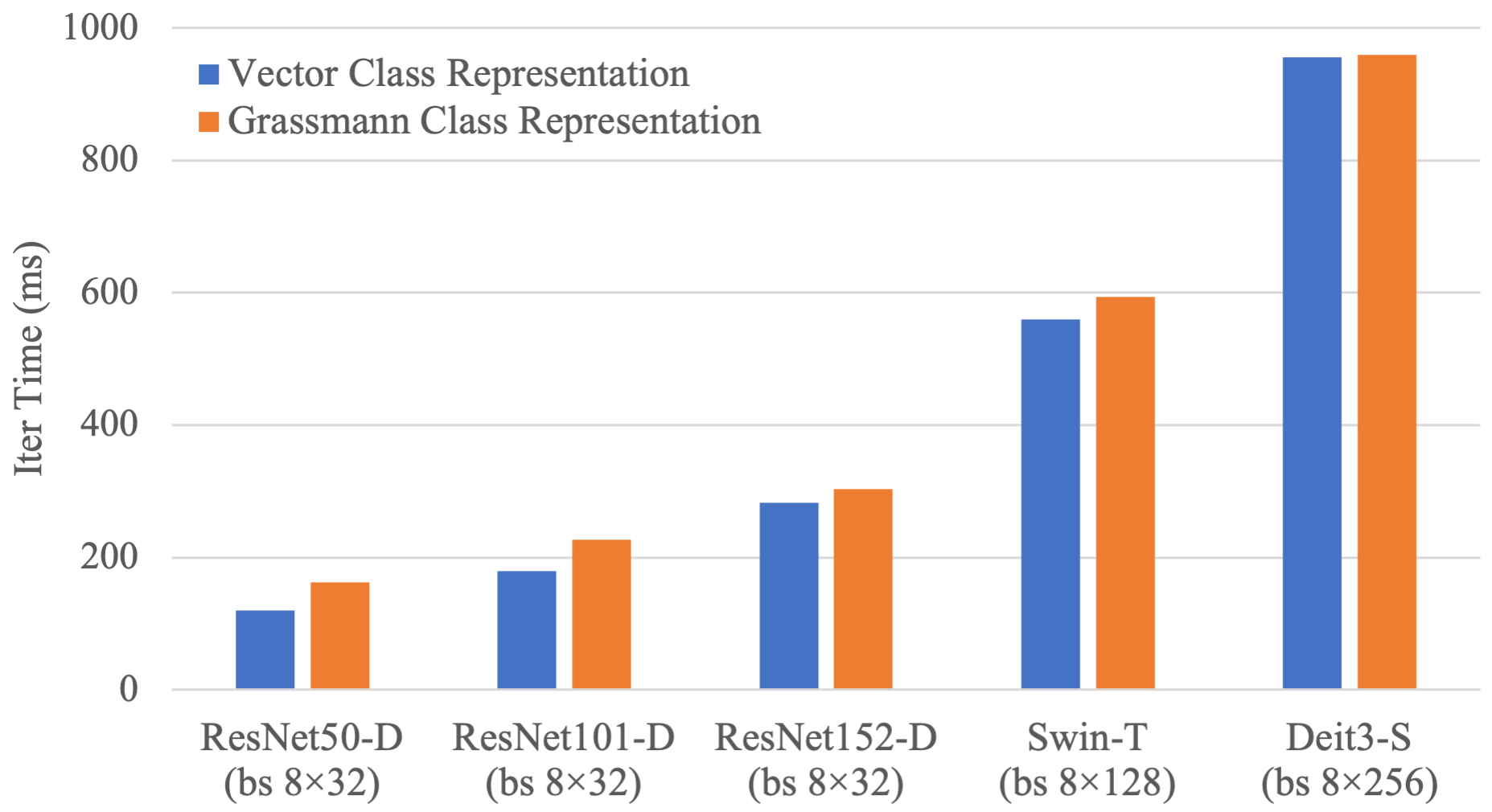}
  \caption{\small
    Compare the iteration time (in \emph{ms}) between vector class representation and Grassmann class representation (\(k=8\)) using different network architectures.
    Blue bars are networks with the original vector class representation and the orange bars are networks with the Grassmann class representation.
    The \emph{bs 8\(\times\)32} means that the batch size is 256, distributed across 8 GPUs and there are 32 samples per GPU.
  }\label{fig:itertime}
\end{figure}

\section{More Visualizations on Principal Angles}

Due to limited space, we only showed the visualization of the maximum and the minimum principal angles in Fig.~3.
Here, we illustrate all eight principal angles in the GCR (\(k=8\)) setting in Fig.~\ref{fig:8principalangles}.

\begin{figure}[t]
  \centering
  \makebox[0\textwidth][c]{
    \begin{subfigure}[t]{0.34\linewidth}
      \centering
      \includegraphics[height=3cm]{output0.png}
      \caption{~}
    \end{subfigure}%
    \begin{subfigure}[t]{0.33\linewidth}
      \centering
      \includegraphics[height=3cm]{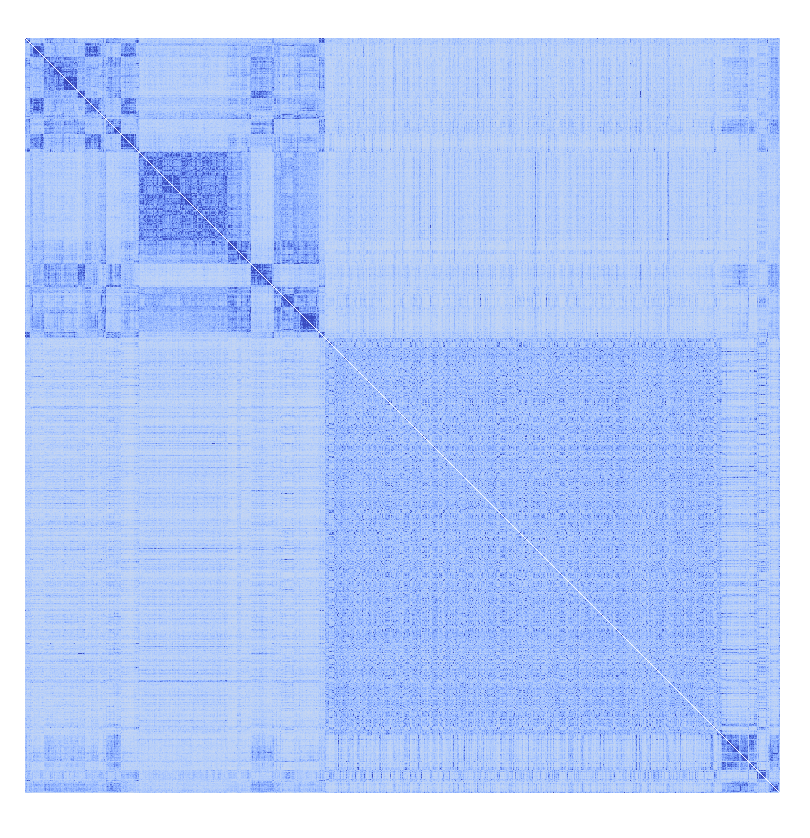}
      \caption{~}
    \end{subfigure}
    \begin{subfigure}[t]{0.32\linewidth}
      \centering
      \includegraphics[height=3cm]{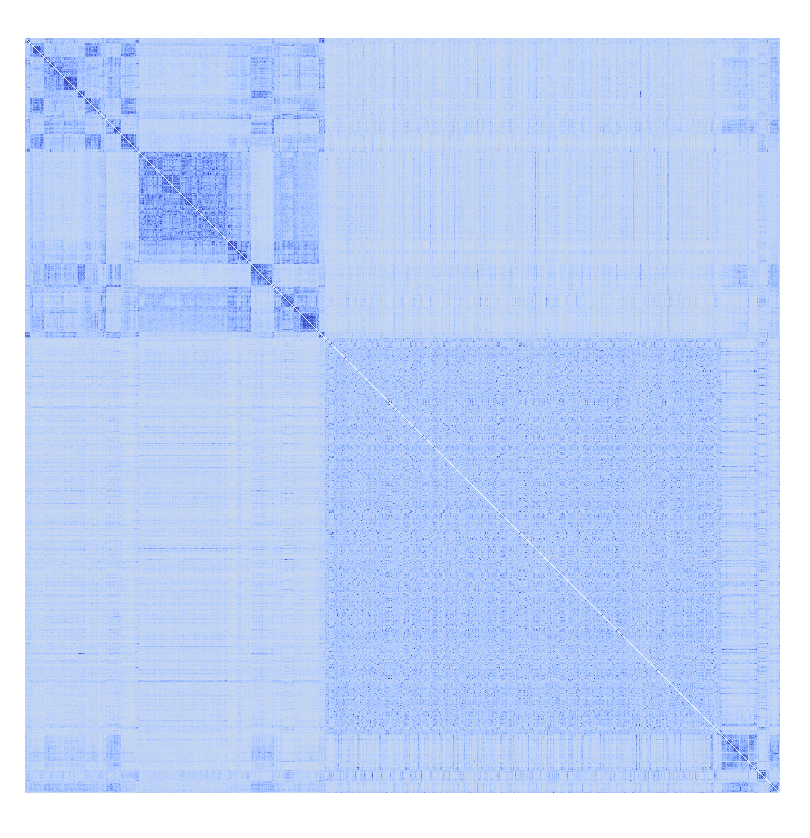}
      \caption{~}
    \end{subfigure}

  }
  \\
  \centering
  \makebox[0\textwidth][c]{
    \begin{subfigure}[t]{0.34\linewidth}
      \centering
      \includegraphics[height=3cm]{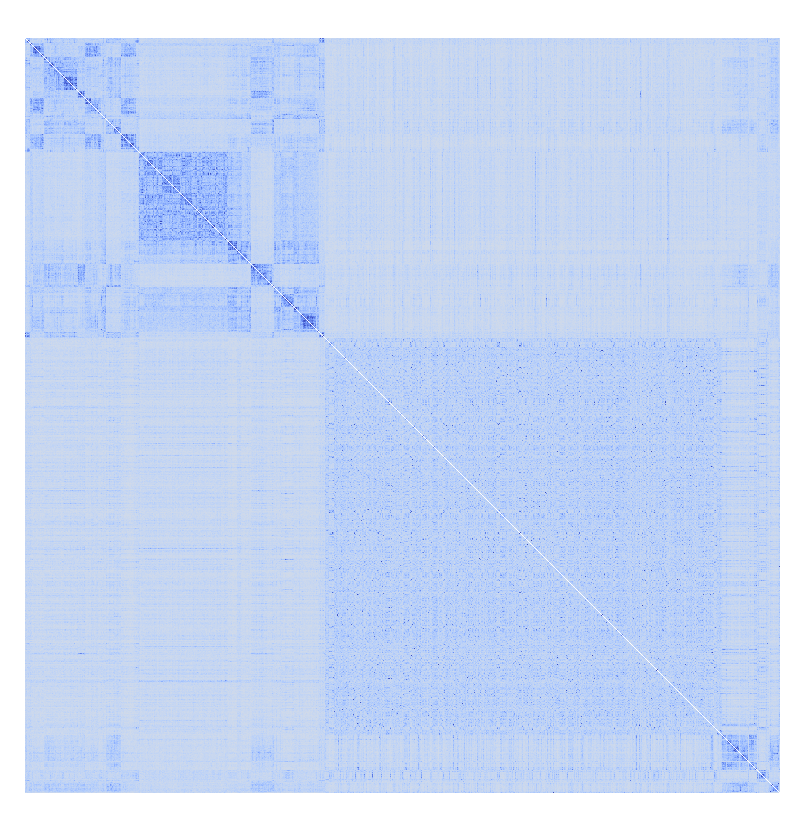}
      \caption{~}
    \end{subfigure}%
    \begin{subfigure}[t]{0.33\linewidth}
      \centering
      \includegraphics[height=3cm]{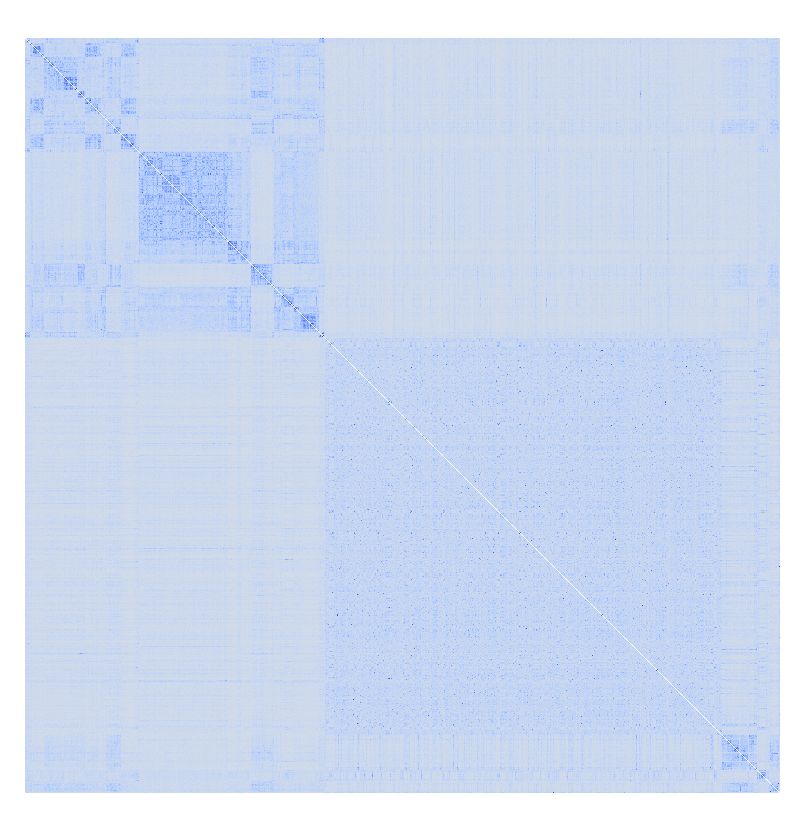}
      \caption{~}
    \end{subfigure}
    \begin{subfigure}[t]{0.32\linewidth}
      \centering
      \includegraphics[height=3cm]{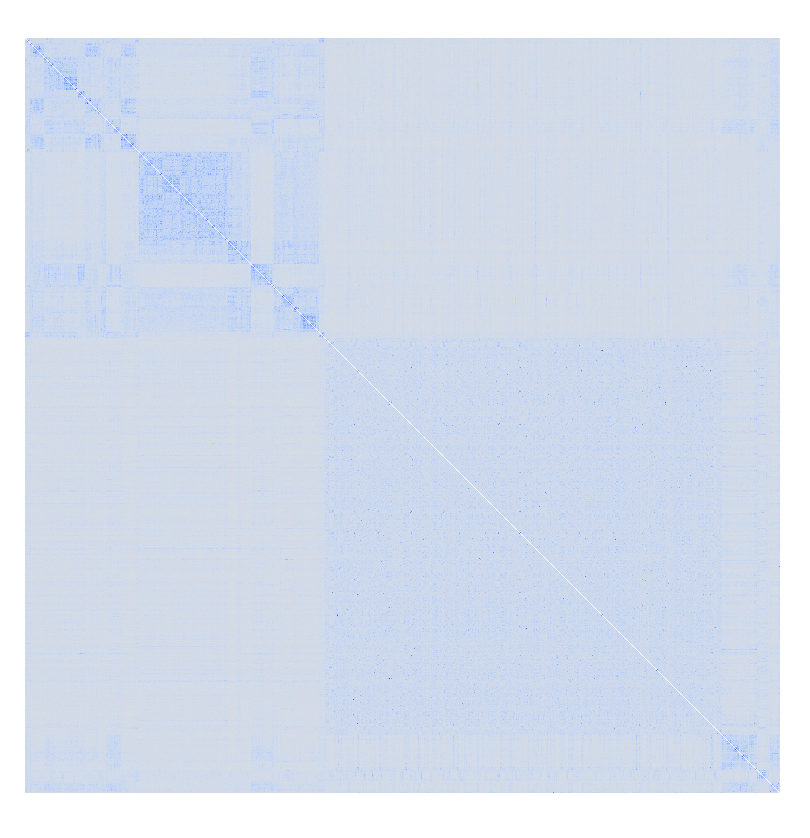}
      \caption{~}
    \end{subfigure}

  }
  \\
  \centering
  \makebox[0\textwidth][c]{
    \begin{subfigure}[t]{0.34\linewidth}
      \centering
      \includegraphics[height=3cm]{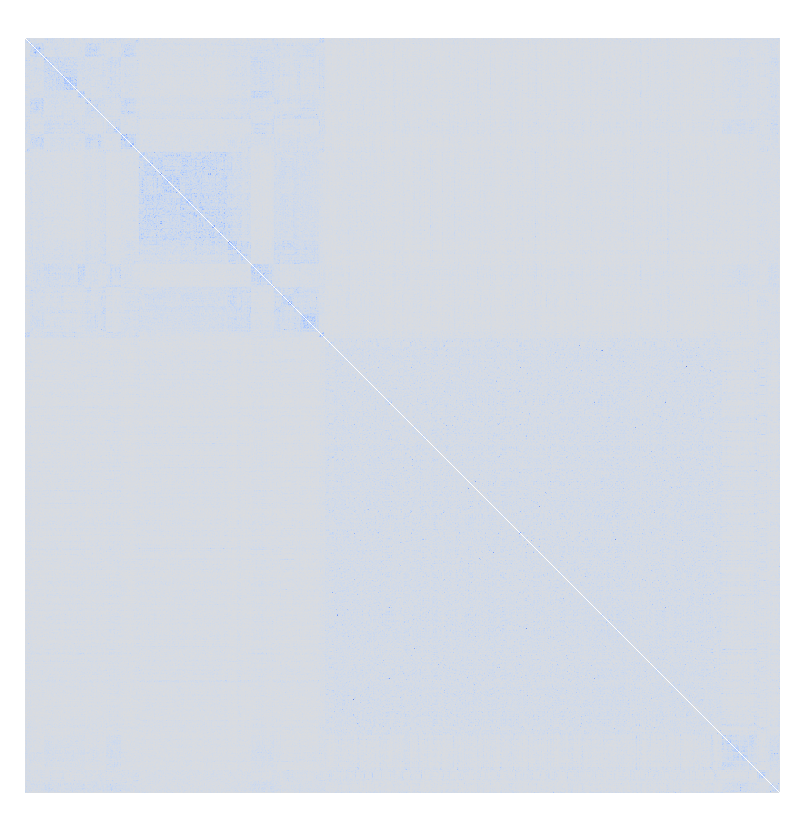}
      \caption{~}
    \end{subfigure}%
    \begin{subfigure}[t]{0.33\linewidth}
      \centering
      \includegraphics[height=3cm]{output-1.png}
      \caption{~}
    \end{subfigure}
    ~~~~~~~~~~~~~~~~~~~~~~~~~~~~~\,
  }
  \caption{\small
    Each sub-figure is a heatmap of \(1000\times1000\) grids.
    The color at the \(i\)-th row and the \(j\)-th column represent an angle between class \(i\) and class \(j\) in ImageNet-1K.
    Pairwise smallest
    (a)-(h) are pairwise principal angles between 8-dimensional class subspaces of a ResNet50-D model.
    (a) shows the smallest principal angels between any pair of classes;
    (b) shows the second smallest principal angles between any pair of classes;
    \emph{etc.}
    Deeper Blue colors mean that their angles are smaller.
    Grayish colors mean the angles are close to \(90^{\circ}\).
    Best viewed on screen with colors.
  }\label{fig:8principalangles}
\end{figure}

\section{Details on the Intra-Class Variability}

In Section~5.2, we introduced the intra-class variability which is defined as the mean pairwise angles (in degrees) between features within the same class and
then averaged over all classes.
For models trained on the ImageNet-1K, we randomly sampled 200K training samples and use their global-centered feature to compute the intra-class variability.
Suppose the set of global-centered features of class \(i\) is \(F_i\), then
\begin{equation}
  \variability \coloneqq \frac{1}{C\left|F_i\right|^2}\sum_{i=1}^{C}\sum_{\vx_j,\vx_k\in F_i}\angle(\vx_j, \vx_k)
\end{equation}
where \(C\) is the number of classes, \(\angle (\cdot,\cdot)\) is the angle (in degree) between two vectors,
and \(\left|F_i\right|\) is the cardinality of the set \(F_i\).

\section{Details on Transfer Datasets}

In this section, we give the details of the datasets that are used in the feature transferability experiments.
They are CIFAR-10~\cite{krizhevsky2009learning}, CIFAR-100~\cite{krizhevsky2009learning}, Food-101~\cite{bossard2014food}, Oxford-IIIT Pets~\cite{parkhi2012cats}, Stanford Cars~\cite{krause2013collecting}, and Oxford 102 Flowers~\cite{nilsback2008automated}.
The number of classes and the sizes of the training set and testing set are shown in Tab.~\ref{tab:trans_data}.

\begin{table}[t]
  \caption{
    Details of the transfer datasets.
    The number of classes, the size of the training set and the testing set (or the validation set if no testing set or label of the testing set is not available),
    and the metric used to report the accuracies.
  }\label{tab:trans_data}
  \centering
  \small
  \setlength{\tabcolsep}{1.5pt}
  \begin{tabular}{lccccc}
    \toprule
    Dataset                                         & Classes & Size (Train/Test)   & Accuracy       \\
    \midrule
    CIFAR-10~\cite{krizhevsky2009learning},         & \(10\)  & \(50000\)/\(10000\) & top-1          \\
    CIFAR-100~\cite{krizhevsky2009learning},        & \(10\)  & \(50000\)/\(10000\) & top-1          \\
    Food-101~\cite{bossard2014food}                 & \(101\) & \(75750\)/\(25250\) & top-1          \\
    Oxford-IIIT Pets~\cite{parkhi2012cats}          & \(37\)  & \(3680/3369\)       & mean per-class \\
    Stanford Cars~\cite{krause2013collecting}       & \(196\) & \(8144\)/\(8041\)   & top-1          \\
    Oxford 102 Flowers~\cite{nilsback2008automated} & \(102\) & \(6552\)/\(818\)    & mean per-class \\
    \bottomrule
  \end{tabular}
\end{table}

\section{Details on Linear SVM Hyperparameter}

In Tab.~3, we used five-fold cross-validation on the training set to determine the regularization parameter of the linear SVM.
The parameter is searched in the set \([0.1, 0.2, 0.5, 1, 2, 5, 10, 15, 20]\).
Tab.~\ref{tab:svmhyper} lists the selected regularization parameter of each setting.
Both the cross-validation procedure and the SVM are implemented using the \emph{sklearn} package.
As a pre-processing step, the features are divided by the average norm of the respective training set, so that SVMs are easier to converge.
The max iteration of SVM is set to 10,000.

\begin{table}[t]
  \caption{
    Regularization hyperparameter of SVM used in the linear feature transfer experiment.
    The hyperparameters are determined by five-fold cross-validation on the training sets.
    \emph{C10} means CIFAR10 and \emph{C100} means CIFAR100.
  }\label{tab:svmhyper}
  \centering
  \small
  \setlength{\tabcolsep}{1.5pt}
  \begin{tabular}{lccccccc}
    \toprule
    \multicolumn{2}{c}{\textbf{Setting}}        & \multicolumn{6}{c}{\textbf{Hyper-Parameter of SVM}}                                                            \\
    Name                                        & \(k\)                                               & C10~    & C100~~  & ~Food   & Pets    & Cars   & Flowers \\
    \cmidrule(r){1-2}\cmidrule(l){3-8}
    Softmax~\cite{bridle1990probabilistic}      &                                                     & \(10\)  & \(5\)   & \(10\)  & \(0.5\) & \(5\)  & \(10\)  \\
    CosineSoftmax~\cite{kornblith2021why}       &                                                     & \(1\)   & \(1\)   & \(1\)   & \(2\)   & \(1\)  & \(2\)   \\
    LabelSmoothing~\cite{szegedy2016rethinking} &                                                     & \(10\)  & \(5\)   & \(5\)   & \(2\)   & \(5\)  & \(10\)  \\
    Dropout~\cite{srivastava2014dropout}        &                                                     & \(5\)   & \(5\)   & \(10\)  & \(2\)   & \(15\) & \(5\)   \\
    Sigmoid~\cite{beyer2020we}                  &                                                     & \(10\)  & \(5\)   & \(5\)   & \(2\)   & \(10\) & \(15\)  \\
    \cmidrule(r){1-2}\cmidrule(l){3-8}
    \multirow{5}{*}{GCR (Ours)}                 & 1                                                   & \(1\)   & \(0.5\) & \(0.5\) & \(1\)   & \(1\)  & \(2\)   \\
                                                & 4                                                   & \(1\)   & \(0.5\) & \(0.5\) & \(5\)   & \(2\)  & \(5\)   \\
                                                & 8                                                   & \(1\)   & \(0.5\) & \(0.5\) & \(1\)   & \(1\)  & \(2\)   \\
                                                & 16                                                  & \(1\)   & \(0.5\) & \(0.5\) & \(1\)   & \(2\)  & \(5\)   \\
                                                & 32                                                  & \(1\)   & \(1\)   & \(0.5\) & \(2\)   & \(2\)  & \(2\)   \\
    \cmidrule(r){1-2}\cmidrule(l){3-8}
    Swin-T~\cite{liu2021swin}                   &                                                     & \(1\)   & \(1\)   & \(1\)   & \(1\)   & \(2\)  & \(5\)   \\
    Swin-T GCR                                  & 8                                                   & \(0.5\) & \(1\)   & \(1\)   & \(2\)   & \(2\)  & \(10\)  \\
    Deit3-S~\cite{touvron2022deit}              &                                                     & \(2\)   & \(2\)   & \(2\)   & \(2\)   & \(5\)  & \(10\)  \\
    Deit3-S GCR                                 & 8                                                   & \(1\)   & \(1\)   & \(0.5\) & \(0.5\) & \(2\)  & \(2\)   \\
    \bottomrule
  \end{tabular}
\end{table}

\section{Feature Transfer Using KNN}

In Tab.~3, we have tested the feature transferability using linear SVM\@.
Here we provide transfer results by KNN in Tab.~\ref{tab:knntransfer}.
The hyperparameter \(K\) in KNN is determined by five-fold cross-validation on the training set.
The candidate values are \(1, 3, 5, \dots, 49\).
Our GCR demonstrates the best performance both on both CNNs and Vision Transformers.
For the ResNet50-D backbone, Grassmann with \(k=32\) has a better performance in both classification accuracy and transferability than all the baseline methods.
On Swin-T, our method surpasses the original Swin-T by 2.76\% on average.
On Deit3-S, our method is 13.81\% points better than the original Deit3-S.
The experiments on KNN reinforced our conclusion that GCR improves large-scale classification accuracy and feature transferability simultaneously.

\begin{table*}[t]
  \caption{Linear transfer using KNN for different losses and different backbones.
    All model weights are pre-trained on ImageNet-1K.
    In the first two sections, ResNet50-D is used as the backbone, and in the last section, Swin-T and Deit3-S are tested.
  }\label{tab:knntransfer}
  \centering
  \small
  \begin{tabular}{l@{}c|c@{}c@{}ccccc}
    \toprule
    \multicolumn{2}{c}{\textbf{Setting}}        & \multicolumn{6}{c}{\textbf{Linear Transfer (KNN)}}                                                                                          \\
    Name                                        & \(k\)                                              & CIFAR10~  & CIFAR100~~ & ~Food     & Pets      & Cars      & Flowers   & \textbf{Avg.} \\
    \cmidrule(r){1-2}\cmidrule(lr){3-8}\cmidrule(l){9-9}
    Softmax~\cite{bridle1990probabilistic}      &                                                    & \(87.24\) & \(56.67\)  & \(57.62\) & \(90.46\) & \(27.99\) & \(84.23\) & \(67.37\)     \\
    CosineSoftmax~\cite{kornblith2021why}       &                                                    & \(85.09\) & \(51.40\)  & \(47.46\) & \(87.14\) & \(22.34\) & \(70.54\) & \(60.66\)     \\
    LabelSmoothing~\cite{szegedy2016rethinking} &                                                    & \(86.46\) & \(54.24\)  & \(52.63\) & \(90.60\) & \(24.56\) & \(79.22\) & \(64.62\)     \\
    Dropout~\cite{srivastava2014dropout}        &                                                    & \(86.43\) & \(53.83\)  & \(52.59\) & \(89.81\) & \(24.28\) & \(77.51\) & \(64.08\)     \\
    Sigmoid~\cite{beyer2020we}                  &                                                    & \(87.96\) & \(58.27\)  & \(57.47\) & \(90.54\) & \(27.22\) & \(84.47\) & \(67.66\)     \\
    \cmidrule(r){1-2}\cmidrule(lr){3-8}\cmidrule(l){9-9}
    \multirow{5}{*}{GCR (Ours)}                 & 1                                                  & \(86.65\) & \(52.72\)  & \(46.83\) & \(86.73\) & \(21.58\) & \(66.99\) & \(60.25\)     \\
                                                & 4                                                  & \(86.62\) & \(54.16\)  & \(51.34\) & \(88.28\) & \(26.39\) & \(73.11\) & \(63.32\)     \\
                                                & 8                                                  & \(86.84\) & \(55.64\)  & \(53.31\) & \(87.93\) & \(27.97\) & \(79.22\) & \(65.15\)     \\
                                                & 16                                                 & \(87.34\) & \(57.31\)  & \(55.26\) & \(89.64\) & \(29.64\) & \(84.60\) & \(67.30\)     \\
                                                & 32                                                 & \(86.96\) & \(56.39\)  & \(56.88\) & \(89.75\) & \(30.31\) & \(87.04\) & \(67.89\)     \\
    \cmidrule(r){1-2}\cmidrule(lr){3-8}\cmidrule(l){9-9}
    Swin-T~\cite{liu2021swin}                   &                                                    & \(90.59\) & \(59.27\)  & \(62.46\) & \(90.27\) & \(28.65\) & \(87.29\) & \(69.76\)     \\
    Swin-T GCR                                  & 8                                                  & \(91.38\) & \(62.04\)  & \(66.57\) & \(91.91\) & \(32.40\) & \(90.83\) & \(72.52\)     \\
    Deit3-S~\cite{touvron2022deit}              &                                                    & \(86.04\) & \(50.54\)  & \(45.47\) & \(88.12\) & \(18.22\) & \(63.69\) & \(58.68\)     \\
    Deit3-S GCR                                 & 8                                                  & \(91.64\) & \(63.80\)  & \(65.18\) & \(91.80\) & \(33.88\) & \(88.63\) & \(72.49\)     \\
    \bottomrule
  \end{tabular}
\end{table*}

\end{document}

%% file: math_commands.tex
\usepackage{amsmath,amsfonts,bm}

\def\eqref#1{equation~(\ref{#1})}
\def\Eqref#1{Equ.~(\ref{#1})}

\def\1{\bm{1}}

\def\vzero{{\bm{0}}}

\def\vb{{\bm{b}}}

\def\vd{{\bm{d}}}

\def\vg{{\bm{g}}}

\def\vr{{\bm{r}}}
\def\vs{{\bm{s}}}
\def\vt{{\bm{t}}}

\def\vw{{\bm{w}}}
\def\vx{{\bm{x}}}

\def\mD{{\bm{D}}}

\def\mG{{\bm{G}}}

\def\mI{{\bm{I}}}

\def\mM{{\bm{M}}}

\def\mS{{\bm{S}}}

\def\mU{{\bm{U}}}
\def\mV{{\bm{V}}}
\def\mW{{\bm{W}}}
\def\mX{{\bm{X}}}

\def\mSigma{{\bm{\Sigma}}}

\DeclareMathAlphabet{\mathsfit}{\encodingdefault}{\sfdefault}{m}{sl}
\SetMathAlphabet{\mathsfit}{bold}{\encodingdefault}{\sfdefault}{bx}{n}

\def\gG{{\mathcal{G}}}

\def\gO{{\mathcal{O}}}

\newcommand{\R}{\mathbb{R}}

\newcommand{\variability}{\mathrm{variability}}
\newcommand{\proj}{\mathrm{proj}}

\newcommand{\St}{\mathrm{St}}
\newcommand{\SO}{\mathrm{SO}}

\DeclareMathOperator{\qf}{qf}